%% file: arxiv.tex
\documentclass[10pt,twocolumn,letterpaper]{article}

\usepackage[pagenumbers]{cvpr}      %

\usepackage{graphicx}
\usepackage{amsmath}
\usepackage{amssymb}
\usepackage{booktabs}

\usepackage{url}            %
\usepackage{amsfonts}       %
\usepackage{nicefrac}       %
\usepackage{microtype}      %
\usepackage{xcolor}         %
\usepackage{multirow}
\usepackage{algorithm}
\usepackage{algorithmic}
\usepackage[inline]{enumitem}  %
\usepackage{subcaption}
\usepackage[normalem]{ulem}
\usepackage[group-separator={,},group-minimum-digits={3}]{siunitx} %
\usepackage{mathtools} %
\usepackage{amsthm}
\usepackage{makecell}
\usepackage{colortbl}
\usepackage{nicefrac}
\usepackage{tabularx}

\definecolor{cvprblue}{rgb}{0.21,0.49,0.74}
\usepackage[pagebackref,breaklinks,colorlinks,allcolors=cvprblue]{hyperref}

\usepackage[capitalize]{cleveref}
\crefname{section}{Sec.}{Secs.}
\Crefname{section}{Section}{Sections}
\Crefname{table}{Table}{Tables}
\crefname{table}{Tab.}{Tabs.}

\newcommand\lft{\mathopen{}\left}
\newcommand\rgt{\aftergroup\mathclose\aftergroup{\aftergroup}\right}

\newcommand{\topic}[2][0.5mm]{\vspace{#1}\noindent\textbf{#2}}

\begin{document}

\title{Mobile Video Diffusion}

\author{Haitam Ben Yahia\thanks{Equal contribution} \qquad
Denis Korzhenkov\footnotemark[1] \qquad
Ioannis Lelekas \\
Amir Ghodrati \qquad
Amirhossein Habibian \\ \\
{Qualcomm AI Research\thanks{Qualcomm AI Research is an initiative of Qualcomm Technologies, Inc. Snapdragon and Qualcomm branded products are products of Qualcomm Technologies, Inc. and/or its subsidiaries.}} \\
{\tt\small \{hyahia, dkorzhen, ilelekas, ghodrati, ahabibia\}@qti.qualcomm.com} %
}

\input{definitions}

\input{macros}

\maketitle
\input{texts/0_abstract}

\input{texts/1_introduction}
\input{texts/2_related_work}

\input{texts/3_method}

\input{texts/4_experiments}

\input{texts/5_conclusion}

{\small
\bibliographystyle{ieeenat_fullname}
\bibliography{bibliography}
}

\clearpage
\appendix
\maketitlesupplementary
\input{supplementary_chapters/0_results}
\input{supplementary_chapters/1_method_details}

\end{document}

%% file: definitions.tex
\newcommand\ah[1]{\textcolor{green}{AmirH: #1}} %
\newcommand\denis[1]{\textcolor{blue}{Denis: #1}} %
\newcommand\haitam[1]{\textcolor{orange}{Haitam: #1}} %
\newcommand\amirg[1]{\textcolor{red}{AmirG: #1}} %
\newcommand\ioannis[1]{\textcolor{pink}{Ioannis: #1}} %

\newcommand{\unet}{\bm{\epsilon}}
\newcommand{\unetH}{\bm{\epsilon}_H}
\newcommand{\unetHin}{\bm{\epsilon}_H^{in}}
\newcommand{\unetHout}{\bm{\epsilon}_H^{out}}
\newcommand{\unetL}{\bm{\epsilon}_L}
\newcommand{\repin}[1]{\bm{r}_{#1}^{in}}
\newcommand{\repout}[1]{\bm{r}_{#1}^{out}}
\newcommand{\repoutpred}[1]{\hat{\bm{r}}_{#1}^{out}}
\newcommand{\adap}{\bm{\phi}_\theta}

\newcommand\methodname{MobileVD\xspace}
\newcommand\methodnamehires{MobileVD-HD\xspace}
\newcommand\sfv{SF-V\xspace}
\newcommand\jedi{JEDi\xspace}

\newcommand{\cin}{c_{\textrm{in}}}
\newcommand{\cout}{c_{\textrm{out}}}
\newcommand{\cinner}{c_{\textrm{inner}}}
\newcommand{\lhsisolated}{\hspace{2em}&\hspace{-2em}}  %

\newtheorem{theorem}{Theorem}[section]
\newtheorem{corollary}{Corollary}[theorem]
\newtheorem{lemma}[theorem]{Lemma}

%% file: macros.tex
\newcommand{\head}[1]{{\smallskip\noindent\textbf{#1}}}
\newcommand{\alert}[1]{{\color{red}{#1}}}
\newcommand{\sm}{\scriptsize}
\newcommand{\eq}[1]{(\ref{eq:#1})}

\newcommand{\Th}[1]{\textsc{#1}}
\newcommand{\mr}[2]{\multirow{#1}{*}{#2}}
\newcommand{\mc}[2]{\multicolumn{#1}{c}{#2}}
\newcommand{\tb}[1]{\textbf{#1}}
\newcommand{\ch}{\checkmark}

\newcommand{\red}[1]{{\color{red}{#1}}}
\newcommand{\blue}[1]{{\color{blue}{#1}}}
\newcommand{\green}[1]{\color{green}{#1}}

\newcommand{\citeme}[1]{\red{[XX]}}
\newcommand{\refme}[1]{\red{(XX)}}

\newcommand{\fig}[2][1]{\includegraphics[width=#1\linewidth]{fig/#2}}
\newcommand{\figh}[2][1]{\includegraphics[height=#1\linewidth]{fig/#2}}

\newcommand{\tran}{^\top}
\newcommand{\mtran}{^{-\top}}
\newcommand{\zcol}{\mathbf{0}}
\newcommand{\zrow}{\zcol\tran}

\newcommand{\ind}{\mathbbm{1}}
\newcommand{\expect}{\mathbb{E}}
\newcommand{\nat}{\mathbb{N}}
\newcommand{\zahl}{\mathbb{Z}}
\newcommand{\real}{\mathbb{R}}
\newcommand{\proj}{\mathbb{P}}
\newcommand{\prob}{\mathbf{Pr}}
\newcommand{\normal}{\mathcal{N}}

\newcommand{\mif}{\textrm{if}\ }
\newcommand{\other}{\textrm{otherwise}}
\newcommand{\minimize}{\textrm{minimize}\ }
\newcommand{\maximize}{\textrm{maximize}\ }
\newcommand{\st}{\textrm{subject\ to}\ }

\newcommand{\id}{\operatorname{id}}
\newcommand{\const}{\operatorname{const}}
\newcommand{\sgn}{\operatorname{sgn}}
\newcommand{\var}{\operatorname{Var}}
\newcommand{\mean}{\operatorname{mean}}
\newcommand{\trace}{\operatorname{tr}}
\newcommand{\diag}{\operatorname{diag}}
\newcommand{\vect}{\operatorname{vec}}
\newcommand{\cov}{\operatorname{cov}}
\newcommand{\sign}{\operatorname{sign}}
\newcommand{\prj}{\operatorname{proj}}

\newcommand{\softmax}{\operatorname{softmax}}
\newcommand{\clip}{\operatorname{clip}}

\newcommand{\defn}{\mathrel{:=}}
\newcommand{\peq}{\mathrel{+\!=}}
\newcommand{\meq}{\mathrel{-\!=}}

\newcommand{\floor}[1]{\left\lfloor{#1}\right\rfloor}
\newcommand{\ceil}[1]{\left\lceil{#1}\right\rceil}
\newcommand{\inner}[1]{\left\langle{#1}\right\rangle}
\newcommand{\norm}[1]{\left\|{#1}\right\|}
\newcommand{\abs}[1]{\left|{#1}\right|}
\newcommand{\frob}[1]{\norm{#1}_F}
\newcommand{\card}[1]{\left|{#1}\right|\xspace}
\newcommand{\divg}[2]{{#1\ ||\ #2}}
\newcommand{\diff}{\mathrm{d}}
\newcommand{\der}[3][]{\frac{d^{#1}#2}{d#3^{#1}}}
\newcommand{\pder}[3][]{\frac{\partial^{#1}{#2}}{\partial{#3^{#1}}}}
\newcommand{\ipder}[3][]{\partial^{#1}{#2}/\partial{#3^{#1}}}
\newcommand{\dder}[3]{\frac{\partial^2{#1}}{\partial{#2}\partial{#3}}}

\newcommand{\wb}[1]{\overline{#1}}
\newcommand{\wt}[1]{\widetilde{#1}}

\def\xssp{\hspace{1pt}}
\def\ssp{\hspace{3pt}}
\def\msp{\hspace{5pt}}
\def\lsp{\hspace{12pt}}

\newcommand{\cA}{\mathcal{A}}
\newcommand{\cB}{\mathcal{B}}
\newcommand{\cC}{\mathcal{C}}
\newcommand{\cD}{\mathcal{D}}
\newcommand{\cE}{\mathcal{E}}
\newcommand{\cF}{\mathcal{F}}
\newcommand{\cG}{\mathcal{G}}
\newcommand{\cH}{\mathcal{H}}
\newcommand{\cI}{\mathcal{I}}
\newcommand{\cJ}{\mathcal{J}}
\newcommand{\cK}{\mathcal{K}}
\newcommand{\cL}{\mathcal{L}}
\newcommand{\cM}{\mathcal{M}}
\newcommand{\cN}{\mathcal{N}}
\newcommand{\cO}{\mathcal{O}}
\newcommand{\cP}{\mathcal{P}}
\newcommand{\cQ}{\mathcal{Q}}
\newcommand{\cR}{\mathcal{R}}
\newcommand{\cS}{\mathcal{S}}
\newcommand{\cT}{\mathcal{T}}
\newcommand{\cU}{\mathcal{U}}
\newcommand{\cV}{\mathcal{V}}
\newcommand{\cW}{\mathcal{W}}
\newcommand{\cX}{\mathcal{X}}
\newcommand{\cY}{\mathcal{Y}}
\newcommand{\cZ}{\mathcal{Z}}

\newcommand{\vA}{\mathbf{A}}
\newcommand{\vB}{\mathbf{B}}
\newcommand{\vC}{\mathbf{C}}
\newcommand{\vD}{\mathbf{D}}
\newcommand{\vE}{\mathbf{E}}
\newcommand{\vF}{\mathbf{F}}
\newcommand{\vG}{\mathbf{G}}
\newcommand{\vH}{\mathbf{H}}
\newcommand{\vI}{\mathbf{I}}
\newcommand{\vJ}{\mathbf{J}}
\newcommand{\vK}{\mathbf{K}}
\newcommand{\vL}{\mathbf{L}}
\newcommand{\vM}{\mathbf{M}}
\newcommand{\vN}{\mathbf{N}}
\newcommand{\vO}{\mathbf{O}}
\newcommand{\vP}{\mathbf{P}}
\newcommand{\vQ}{\mathbf{Q}}
\newcommand{\vR}{\mathbf{R}}
\newcommand{\vS}{\mathbf{S}}
\newcommand{\vT}{\mathbf{T}}
\newcommand{\vU}{\mathbf{U}}
\newcommand{\vV}{\mathbf{V}}
\newcommand{\vW}{\mathbf{W}}
\newcommand{\vX}{\mathbf{X}}
\newcommand{\vY}{\mathbf{Y}}
\newcommand{\vZ}{\mathbf{Z}}

\newcommand{\va}{\mathbf{a}}
\newcommand{\vb}{\mathbf{b}}
\newcommand{\vc}{\mathbf{c}}
\newcommand{\vd}{\mathbf{d}}
\newcommand{\ve}{\mathbf{e}}
\newcommand{\vf}{\mathbf{f}}
\newcommand{\vg}{\mathbf{g}}
\newcommand{\vh}{\mathbf{h}}
\newcommand{\vi}{\mathbf{i}}
\newcommand{\vj}{\mathbf{j}}
\newcommand{\vk}{\mathbf{k}}
\newcommand{\vl}{\mathbf{l}}
\newcommand{\vm}{\mathbf{m}}
\newcommand{\vn}{\mathbf{n}}
\newcommand{\vo}{\mathbf{o}}
\newcommand{\vp}{\mathbf{p}}
\newcommand{\vq}{\mathbf{q}}
\newcommand{\vr}{\mathbf{r}}
\newcommand{\Vs}{\mathbf{s}}
\newcommand{\vt}{\mathbf{t}}
\newcommand{\vu}{\mathbf{u}}
\newcommand{\vv}{\mathbf{v}}
\newcommand{\vw}{\mathbf{w}}
\newcommand{\vx}{\mathbf{x}}
\newcommand{\vy}{\mathbf{y}}
\newcommand{\vz}{\mathbf{z}}

\newcommand{\vone}{\mathbf{1}}
\newcommand{\vzero}{\mathbf{0}}

\newcommand{\valpha}{{\boldsymbol{\alpha}}}
\newcommand{\vbeta}{{\boldsymbol{\beta}}}
\newcommand{\vgamma}{{\boldsymbol{\gamma}}}
\newcommand{\vdelta}{{\boldsymbol{\delta}}}
\newcommand{\vepsilon}{{\boldsymbol{\epsilon}}}
\newcommand{\vzeta}{{\boldsymbol{\zeta}}}
\newcommand{\veta}{{\boldsymbol{\eta}}}
\newcommand{\vtheta}{{\boldsymbol{\theta}}}
\newcommand{\viota}{{\boldsymbol{\iota}}}
\newcommand{\vkappa}{{\boldsymbol{\kappa}}}
\newcommand{\vlambda}{{\boldsymbol{\lambda}}}
\newcommand{\vmu}{{\boldsymbol{\mu}}}
\newcommand{\vnu}{{\boldsymbol{\nu}}}
\newcommand{\vxi}{{\boldsymbol{\xi}}}
\newcommand{\vomikron}{{\boldsymbol{\omikron}}}
\newcommand{\vpi}{{\boldsymbol{\pi}}}
\newcommand{\vrho}{{\boldsymbol{\rho}}}
\newcommand{\vsigma}{{\boldsymbol{\sigma}}}
\newcommand{\vtau}{{\boldsymbol{\tau}}}
\newcommand{\vupsilon}{{\boldsymbol{\upsilon}}}
\newcommand{\vphi}{{\boldsymbol{\phi}}}
\newcommand{\vchi}{{\boldsymbol{\chi}}}
\newcommand{\vpsi}{{\boldsymbol{\psi}}}
\newcommand{\vomega}{{\boldsymbol{\omega}}}

\newcommand{\rLambda}{\mathrm{\Lambda}}
\newcommand{\rSigma}{\mathrm{\Sigma}}

\newcommand{\vLambda}{\bm{\rLambda}}
\newcommand{\vSigma}{\bm{\rSigma}}

\makeatletter
\newcommand{\vast}[1]{\bBigg@{#1}}
\makeatother

\makeatletter
\newcommand*\bdot{\mathpalette\bdot@{.7}}
\newcommand*\bdot@[2]{\mathbin{\vcenter{\hbox{\scalebox{#2}{$\m@th#1\bullet$}}}}}
\makeatother

\makeatletter
\DeclareRobustCommand\onedot{\futurelet\@let@token\@onedot}
\def\@onedot{\ifx\@let@token.\else.\null\fi\xspace}

\def\Wlog{\emph{W.l.o.g}\onedot}
\def\eg{\emph{e.g}\onedot} \def\Eg{\emph{E.g}\onedot}
\def\ie{\emph{i.e}\onedot} \def\Ie{\emph{I.e}\onedot}
\def\cf{\emph{cf}\onedot} \def\Cf{\emph{Cf}\onedot}
\def\etc{\emph{etc}\onedot} \def\vs{\emph{vs}\onedot}
\def\wrt{w.r.t\onedot} \def\dof{d.o.f\onedot} \def\aka{a.k.a\onedot}
\def\etal{\emph{et al}\onedot}
\makeatother

%% file: texts/0_abstract.tex
\begin{abstract}
\label{sec:abstract}
Video diffusion models have achieved impressive realism and controllability but are limited by high computational demands, restricting their use on mobile devices. This paper introduces the first mobile-optimized video diffusion model. Starting from a spatio-temporal UNet from Stable Video Diffusion (SVD), we reduce memory and computational cost by reducing the frame resolution, incorporating multi-scale temporal representations, and introducing two novel pruning schema to reduce the number of channels and temporal blocks. Furthermore, we employ adversarial finetuning to reduce the denoising to a single step. Our model, coined as \methodname, is $523\times$ more efficient ($1817.2$ vs. $4.34$ TFLOPs) with a slight quality drop (FVD $149$ vs. $171$), generating latents for a $14 \times 512 \times 256$ px clip in $1.7$ seconds on a Xiaomi-14 Pro.
Our results are available at \url{https://qualcomm-ai-research.github.io/mobile-video-diffusion/}
\end{abstract}

%% file: texts/1_introduction.tex
\section{Introduction}

Video diffusion models are making significant progress in terms of realism, controllability, resolution, and duration of the generated videos. Starting from zero-shot video models~\cite{khachatryan2023text2video, liu2024video, kahatapitiya2025object, esser2023structure}, which deploy pretrained image diffusion models to generate consistent frames, \eg, through cross-frame attention, modern video diffusion models rely on spatio-temporal denoising architectures, \ie, 3D UNets~\cite{blattmann2023align, blattmann_stable_2023, guoanimatediff} or 3D DiTs~\cite{opensora,opensora_plan, yang2024cogvideox}. This involves inflating image denoising models by adding temporal transformers and convolutions to denoise multiple frames simultaneously. Despite their impressive generation qualities, spatio-temporal denoising architectures demand high memory and computational power, which limits their usage to clouds with high-end GPUs. This hinders the wide adoption of video generation technology for many applications that require generating content locally on mobile devices.

\input{figures/cover_fig}

Prior work on accelerating video diffusion models has mostly focused on reducing the number of sampling steps~\cite{wang2024animatelcm, zhang_svf_2024}. By extending the consistency models~\cite{song2023consistency} and adversarial distillation~\cite{sauer2024fast} to video diffusion models, they managed to reduce the number of denoising steps from 25 to only 4~\cite{wang2024animatelcm} and 1 step~\cite{zhang_svf_2024}, which tremendously accelerates video generation. However, step distillation alone does not reduce the memory usage of the model, which is the key setback in deploying video diffusion models on mobile devices.

This paper is the first attempt to build video diffusion models for mobile. Starting from the spatio-temporal UNet from Stable Video Diffusion (SVD)~\cite{blattmann2023align, blattmann_stable_2023}, as a representative for a variety of video diffusion models, we conduct a series of optimizations to reduce the size of activations to build a \textit{mobile-friendly UNet}: driven by the lower resolution needs for user-generated content on phones, we opt for using a smaller latent space sufficient for generating $512\times256$ px frames. Instead of preserving the number of frames throughout the denoising UNet, we introduce additional temporal down- and up-sampling operations to extend the multi-scale representation both in space and time, which reduces the memory and computational cost with minimal loss in quality. Moreover, we discuss how naive visual conditioning through cross-attention leads to significant computational overhead that can be avoided without damaging visual quality.

We further accelerate the mobile-friendly UNet by reducing its parameters using a novel channel compression schema, coined \textit{channel funneling}, and a novel technique to prune the temporal transformers and temporal residual blocks from the UNet. Finally, following~\citet{zhang_svf_2024}, we reduce the number of denoising steps to a single step using adversarial finetuning. This results in the first mobile video diffusion model called \methodname, which is $523\times$ more efficient ($1817.2$ vs. $4.34$ TFLOPs) at a slightly worse quality in terms of FVD ($149$ vs. $171$) as reported in~\cref{fig:flops_vs_fvd_main}. \methodname generates the latents for a $14 \times 512 \times 256$ px clip in $1.7$ seconds on a Xiaomi 14-Pro smartphone which uses a Qualcomm Snapdragon\textsuperscript{\textregistered} 8 Gen 3 Mobile Platform.

%% file: figures/cover_fig.tex
\begin{figure}[t]
\centering
\includegraphics[width=0.9\columnwidth]{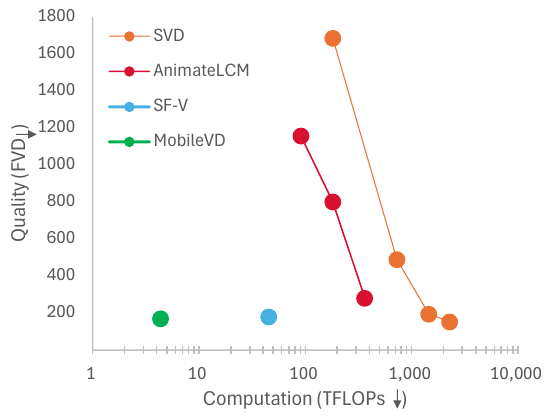}
\caption{
\topic[0mm]{Quality-efficiency trade-off.}
Our \methodname accelerate SVD by $523\times$ (in FLOPs) with a slight decrease in the generation qualities (in FVD) reaching to a better quality vs. efficiency trade-off than alternatives.
}
\vspace{-1.5 em}
\label{fig:flops_vs_fvd_main}
\end{figure}

%% file: texts/2_related_work.tex
\section{Related work}

\topic{Video generation.}
Fueled by advancements in generative image modeling using diffusion models, there has been notable progress in the development of generative video models~\cite{ho2022imagen, blattmann2023align, blattmann_stable_2023,  zhou2022magicvideo, girdhar2023emu, jin_pyramidal_2024, openai2024, kling2024, runway2024, mochi2024}. These video models generally evolve from image models by incorporating additional temporal layers atop spatial blocks or by transforming existing 2D blocks into 3D blocks to effectively capture motion dynamics within videos. Although these advancements have paved the way for the generation of high-resolution videos, the significant computational demands make them impractical for use on low-end devices. In this work, we address this by optimizing a representative of video generative model, SVD~\cite{blattmann_stable_2023}, to make it accessible to a broader range of consumer-graded devices.

\topic{Diffusion optimization.}
The problem of making diffusion models efficient naturally consists of the following two parts: (i) reducing the number of denoising steps and (ii) decreasing the latency and memory footprint of each of those steps. Reducing number of steps is achieved by using higher-order solvers~\cite{lu2022dpm, lu2022dpmpp, zhang2022fast}, distilling steps to a reduced set using progressive step distillation~\cite{salimans2022progressive, li2023snapfusion, meng2023distillation}, straightening the ODE trajectories using Rectified Flows~\cite{liu2022rectified, liu2022flow, zhu2025slimflow}, mapping noise directly to data with consistency models~\cite{song2023consistency, song2023improved, lu2024simplifying}, and using adversarial training~\cite{wang2022diffusion, sauer2025adversarial, sauer2024fast, zhang_svf_2024}. To decrease computational cost of each step, research has been done in weight quantization~\cite{he2024ptqd, pandey2023softmax, shang2023post} and pruning~\cite{choi2023squeezinglargescalediffusionmodels, li2023snapfusion} as well as architectural optimization of the denoiser~\cite{dockhorn2023distilling, kim2023bk, habibian2024clockwork, xie2024sana}. 
In this work, following~\cite{zhang_svf_2024} we reduce number of steps to one using adversarial training and optimize the UNet denoiser using multiple novel techniques.

\topic{On-device generation.}
On-device generation has attracted interest due to its ability to address privacy concerns associated with cloud-based approaches. There have been advancements in running text-to-image generation on mobile devices and NPUs~\cite{zhao2023mobilediffusion, choi2023squeezinglargescalediffusionmodels, chen2023speed, li2023snapfusion, castells2024edgefusion}. While there has been progress in the video domain with fast zero-shot video editing models~\cite{zhang2024fastvideoedit, kara2024rave, wu2024fairy}, there have been no attempts to implement spatiotemporal video generative models on device due to their high computational and memory requirements, which are challenging to meet within on-device constraints. In this work, we propose the first on-device model based on SVD~\cite{blattmann_stable_2023} image-to-video model.

%% file: texts/3_method.tex
\section{Mobile Video Diffusion}
\label{sec:method}

In this section, we propose a series of optimizations to obtain a fast and lightweight version of an off-the-shelf image-to-video diffusion model~\cite{blattmann_stable_2023}, suitable for on-device deployment.
\subsection{Preliminaries}
\label{subsec:preliminaries}
\topic{Base model.}
We adopt Stable Video Diffusion (SVD)~\cite{blattmann_stable_2023} as the base model for optimization. 
The SVD released checkpoint\footnote{\url{https://huggingface.co/stabilityai/stable-video-diffusion-img2vid/tree/9cf024d}} is an image-to-video model that by default generates $14$ frames at the resolution $1024 \times 576$ with $25$ sampling steps. To generate a video from an image, the input image is first mapped into a latent code of resolution $128\times72$ using a Variational Auto-Encoder (VAE). Then it is duplicated $14$ times and concatenated with a noise latent of spatiotemporal resolution $14\times128\times72$. The combined latent is then denoised by a conditional UNet through an iterative process. Additionally, the input image is encoded with a CLIP image embedding for use in cross-attention layers~\cite{pmlr-v139-radford21a}.
The UNet denoiser consists of four downsampling blocks, one middle block and four upsampling blocks. To handle the dynamics of video sequences, the model employs temporal blocks after spatial blocks. Notably, up- and downsampling is conducted across spatial axes only, and temporal resolution of the latent is kept constant to $14$ throughout the UNet. 
The UNet denoiser as-is is too resource-intensive for on-device use, requiring \num{45.43} TFLOPs and \num{376} ms per denoising step on a high-end A100 GPU. Using FLOPs and latency as proxy metrics, we propose a series of optimizations to make the model suitable for on-device deployment.

\topic{Adversarial finetuning.}
In addition to the high cost of the UNet, the iterative sampling process in video diffusion models further slows them down. For example, with the conventional 25-step sampling it takes 11 seconds to generate a video  on a high-end A100 GPU. To reduce the cost, we follow \sfv framework~\cite{zhang_svf_2024} and use adversarial finetuning, enabling our models to generate videos in a single forward pass. 
Namely, we initialize the discriminator feature extractor with an encoder part of the denoising UNet and do not train it. 
After each block of this backbone, the extracted feature maps are passed to the two projection discriminator heads, one spatial and one temporal~\cite{miyato_cgans_2018}.
The heads also use the frame index and the CLIP embedding of the conditional image as input.
At each training step, we apply the pseudo-Huber (Charbonnier)~\cite{charbonnier_1994} and non-saturating adversarial loss~\cite{goodfellow_nips_2014} to the generator output to update its weights.
To regularize the discriminator, the $R_1$ penalty is used~\cite{mescheder_which_2018}. For further details please refer to  the original \sfv work.

\input{figures/optraces/base_v_opt_hvx}

\subsection{Mobile-friendly UNet}\label{sec:lightweight_unet}
Our first modifications to SVD architecture regard optimizations along the latent and feature resolution that affect both GPU and mobile latency. Then we highlight some lossless optimizations that have big effect on mobile latency. These are the first optimizations that allow us to run the UNet on device.

\topic{Low resolution finetuning.}
To satisfy the memory constraints of mobile devices, we decrease the resolution of denoising UNet input to $64 \times 32$ by resizing the conditioning image to $512 \times 256$. 
While the released SVD checkpoint supports multiple resolutions, we found out that the original model demonstrates deteriorated quality for our target spatial resolution as reported in~\cref{tab:sota_fvd}.
Therefore, we  finetuned the diffusion denoiser at our target spatial size. 
With this optimization, computational cost and GPU latency is reduced to \num{8.60} TFLOPS and \num{82} ms respectively, see~\cref{tab:optimizations_fvd}.

\topic{Temporal multiscaling.}
To further reduce computational burden, one might lower the input resolution more heavily. However, this significantly degrades visual quality.
Instead, we can additionally downscale the feature maps by a factor of 2 along either spatial or temporal axis after the first downsampling block of the UNet.
To maintain the same output shape, this is accompanied by the corresponding nearest-neighbor upscaling operation before the last upsampling block. We refer to these two additions as multiscaling.
In terms of computational cost, spatial multiscaling results in a \num{51}\% reduction in FLOPs and \num{33}\% in GPU latency, while temporal multiscaling reduces FLOPs and GPU latency by \num{34}\% and \num{22}\%, respectively. For our final deployed model, we use temporal multiscaling as it offers a better trade-off between quality and efficiency, as reported in~\cref{subsubsec:reslution_unet}.

\topic{Optimizing cross-attention.}
In SVD, each cross-attention layer integrates information from the conditioning image into the generation process. The attention scores are computed similarly to self-attention layers, $\textrm{Attn}\lft(Q,K,V\rgt) = \textrm{softmax}\lft(QK^T/ \sqrt{d}\rgt)V$, but the key and value pair ($K, V$) comes from the context tokens. However, the context in the cross-attention blocks always consists of a single token, namely, the CLIP embedding vector of the conditional image. Consequently, each query token attends to only a single key token.
Therefore, computation of a similarity map $QK^T$ and softmax becomes a no-op, and query and key projection matrices can be removed without any difference in results.
While this loss-less optimization reduces GPU latency only by 7\%, we found that it significantly impacts on-device behavior.
In detail, at target resolution of $512 \times 256$, the model runs out of memory (OOM) on the device without the described modification of cross-attention.
And at a smaller resolution of $128\times128$ this optimization reduces mobile latency by \num{32}\%. %
The gain is attributed to the time-consuming nature of softmax operation on device, as shown in~\cref{fig:base_v_opt}.

\begin{figure}[t]
\centering
\includegraphics[width=1.0\columnwidth]{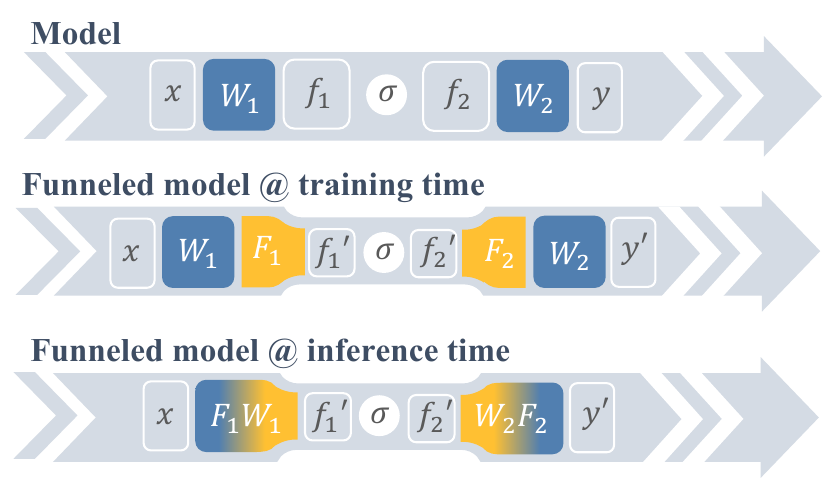}
\caption{
\topic[0mm]{Channel funnels.}
We show an example of channel funnels applied to a couple of layers within the model.
At training time, funnels serve as adaptors reducing model width.
At inference, they are merged with corresponding weight matrices without loss of quality.
}
\label{fig:funnel_pipeline}
\vspace{-2ex}
\end{figure}

\subsection{Channel funnels}
\label{sec:funnel}
Channel size, which refers to the width of neural network layers, is crucial for scaling models. Increasing channel size generally enhances model capacity but also increases the number of parameters and computational cost. Research has focused on compressing the number of channels by either discarding less informative channels through channel pruning~\cite{muralidharan2024compact,fang2024isomorphic} or representing weight matrices as low-rank products of  matrices using truncated singular decomposition~\cite{zhang2015acceleratingVD}.
However, these could have sub-optimal tradeoffs in quality and efficiency when deployed on device. 
While low-rank decomposition is relatively straightforward to implement, it only reduces the number of parameters and computational complexity if the rank is reduced by more than half for feed forward layers. Additionally, not all layers of neural network types benefit equally from low-rank factorization. Moreover, this method does not reduce the size of output feature maps, which can cause significant memory overhead on mobile devices. In this part, we propose \emph{channel funnels}, a straightforward method to reduce the number of channels at inference time, with negligible quality degradation. 
Intuitively, a channel funnel is placed between two affine layers, reducing the intermediate channel dimensionality to save computation. 
Consider two consecutive linear layers  $y = W_2 \sigma\lft(W_1 x\rgt)$ and the non-linear function $\sigma$ in between, where $W_1 \in \mathbb{R}^{\cinner \times \cin}$, $W_2 \in \mathbb{R}^{\cout \times \cinner}$.
We introduce two funnel matrices, $F_1 \in \mathbb{R}^{c' \times \cinner}$ and $F_2 \in \mathbb{R}^{\cinner \times c'}$, where $c' < \cinner$, and rewrite our network as $y' = W_2 F_2 \sigma\lft(F_1 W_1 x\rgt)$.
The $F$-weights, having fewer channels, can be merged during inference with their associated $W$-weights, resulting a weight matrix with smaller inner dimension $c'$, see~\cref{fig:funnel_pipeline}. We refer to the reducing factor of the inner rank of the layers, \ie $\nicefrac{c'}{\cinner}$, as the \emph{fun\sout{nel}-factor}.

\topic{Initialization.}
We propose to use \emph{coupled singular initialization (CSI)} for funnel matrices $F_1$ and $F_2$ that improves the model results, as demonstrated below.
In this method we ignore the non-linearity and consider the effective weight matrix $W_2 F_2 F_1 W_1$ which in practice has rank of $c'$.
For that reason, we aim to use such an initialization which mimics the best $c'$-rank approximation of the original effective matrix.
As Eckart-Young-Mirsky theorem implies, this can be achieved by means of  truncated singular decomposition~\cite{Eckart_Young_1936}.
Let $W_2 W_1 = U \Sigma V^T$ be the singular vector decomposition, and $U_{c'} \Sigma_{c'} V_{c'}^T$ to be its truncated $c'$-rank version.
Then it suffices to set $F_2 = W_2^\dagger U_{c'} \Sigma_{c'}^{1/2}$ and $F_1 = \Sigma_{c'}^{1/2} V_{c'}^T W_1^\dagger$ to obtain $W_2 F_2 F_1 W_1 \approx U_{c'} \Sigma_{c'} V_{c'}^T$, where $\dagger$ means the Moore-Penrose pseudoinverse.

\topic{Training.}
We apply channel funnels in attention layers where query and key embedding $W_q$ and $W_k$ are used to compute the similarity map of $X W_q \left( X W_k \right)^T$.  With funnel matrices $F_q$ and $F_k$ of size $\cinner \times c'$, we modify the aforementioned bilinear map as $X W_q F_q \left( X W_k F_k \right)^T = X W_q F_q F_k^T W_k^T X^T$. 
Similarly, funnels are applied to the pair of value and output projection matrices of a self-attention layer. In our ablations we also show the impact of channel funnel on convolutions in residual blocks. Unless specified otherwise, we use fun-factor of 50\%.

\input{figures/temporal_block_figure}

\input{figures/models_comparison/qualitative_comparison}

\subsection{Temporal block pruning}
\label{sec:temporal_adaptors}
\topic{Motivation.}
The original UNet in the SVD model does not contain 3D convolutions or full spatiotemporal attention layers. 
Instead, to model motion dynamics, SVD incorporates temporal blocks after each spatial block. 
The output of  such group is a linear combination of the spatial and temporal block outputs, $x_s$ and $x_t$ respectively, $\alpha x_s + \left(1-\alpha\right) x_t$, where $\alpha$ is a weight scalar that emphasizes spatial features when higher and temporal features when lower, see~\cref{fig:temporal_block_original_a}.
While this approach leverages image model priors when extending the model to videos, it adds computational cost. Moreover, not all of these blocks are equally important for maintaining quality. Here, we propose a learnable pruning technique to remove less important temporal blocks while minimizing quality degradation.
To this end, for each temporal block we define an \emph{importance} value $q_i, \; 0 \le q_i \le 1$ where $i = 1,\dots,N$ and $N$ is number of temporal blocks. 
The values $\left\{q_i\right\}_i$ are trained to identify the blocks that are the most crucial for model performance.
At inference time, only $n$ blocks with the highest importance $q_i$ are kept where $n$ is the budget chosen in advance. 
In our experiments we found it possible to remove as many as 70\% of all temporal blocks which leads to 14\% reduction in FLOPS as compared to the model with optimizations from~\cref{sec:lightweight_unet} applied.

\topic{Training.} At each training iteration, we sample randomly $n$ blocks which participate in the computational graph. To this end, we define the indicator variable $z_i \in \left\{0, 1\right\}$ where $z_i=1$ if $i$-th block is sampled for participation, and $z_i = 0$ otherwise. The sum of all participants should equal to the budget value \ie $\sum\nolimits_i z_i = n$.
Note that $\mathbb{E} \left[\sum\nolimits_i z_i\right] = \sum\nolimits_i \mathbb{E} \left[z_i\right] = \sum\nolimits_i p_i = n,$ where $p_i$ is the \emph{inclusion probability} of the $i$-th block. We relate $i$-th importance value $q_i$ to the inclusion probability of $p_i$ using the constrained optimization problem { \small%
\begin{equation*}
\min\limits_{c, \{p_i\}_i} \sum\nolimits_i (p_i - c q_i)^2, \quad \textrm{s.t.} \quad \sum\nolimits_i p_i = n, \; 0 \le p_i \le 1, \; c \geq 0.
\end{equation*}%
}%
\footnotetext{Conditioning images are under MIT license \textcopyright~2024 Fu-Yun Wang. \url{https://github.com/G-U-N/AnimateLCM/blob/9a5a314/LICENSE.txt}}
\newcounter{auxFootnote}\setcounter{auxFootnote}{\value{footnote}}
We obtain a closed-form solution of the above optimization with Lagrange multipliers which is differentiable \wrt $q_i$.
In simple words, we find the proper set of inclusion probabilities $p_i \approx c q_i$, with importance values $q_i$ and a proportionality coefficient $c$. 
After obtaining $\left\{p_i\right\}_i$ at each training iteration, we sample $n$ blocks without replacement using Brewer's sampling~\cite{tille_sampling_2006,brewer_simple_1975}. As such sampling is non-differentiable, we employ straight-through estimators (STE)~\cite{bengio2013estimatingpropagatinggradientsstochastic} to enable end-to-end training.
Namely, we define gate $\hat{z}_i$ as STE of the probability $p_i$, \ie $\hat{z}_i = p_i + \texttt{stop\_gradient}\lft( z_i - p_i \rgt)$.
The output of a temporal block is multiplied by this gate which serves as an analog of a Dropout layer~\cite{hinton_dropout_2012}, as~\cref{fig:from_chances_to_gates_b} shows.
For practical aspects of training, please refer to the~\cref{supp:temporal_pruning}. %

%% file: figures/optraces/base_v_opt_hvx.tex
\begin{figure}[t]
\centering
\includegraphics[width=0.99\columnwidth]{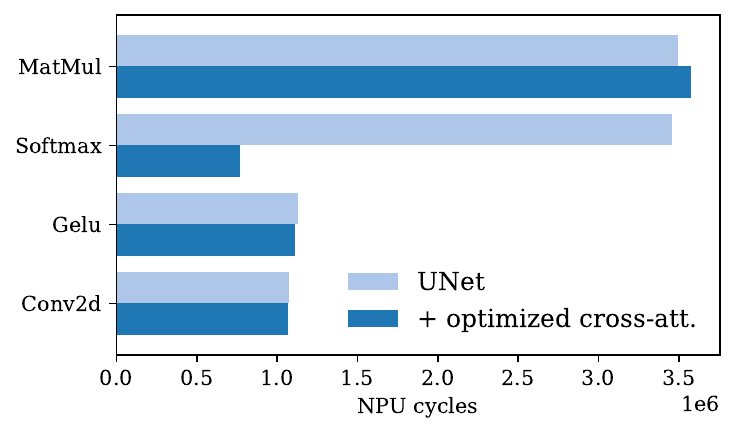} %
\caption{
\topic[0mm]{Effect of optimized cross-attention for a mobile device.}
We show the number of cycles of the top-4 operations on mobile hardware for an input resolution of $128 \times 128$. Note that removing the no-op similarity map computation in cross-attention layers reduces cycles on softmax operations by roughly \num{80}\%.
}
\label{fig:base_v_opt}
\vspace{-2ex}
\end{figure}

%% file: figures/temporal_block_figure.tex
\begin{figure}[t]
\centering
\begin{subfigure}[t]{\linewidth}
\includegraphics[width=\linewidth]{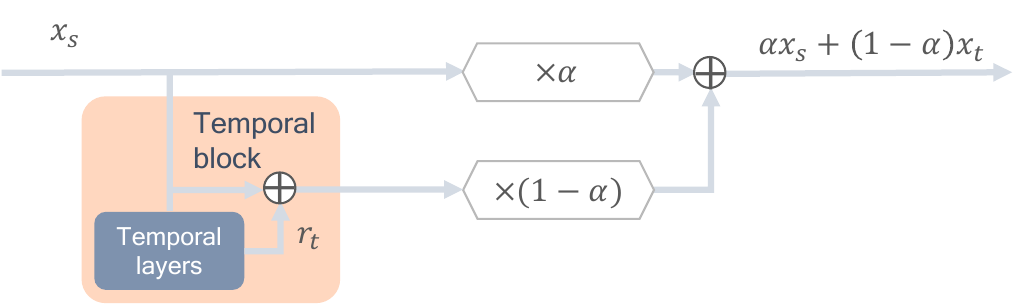}
\subcaption{
Temporal blocks in the original architecture of SVD.
}
\label{fig:temporal_block_original_a}
\end{subfigure}
\begin{subfigure}[t]{\linewidth}
\includegraphics[width=1.08\linewidth]{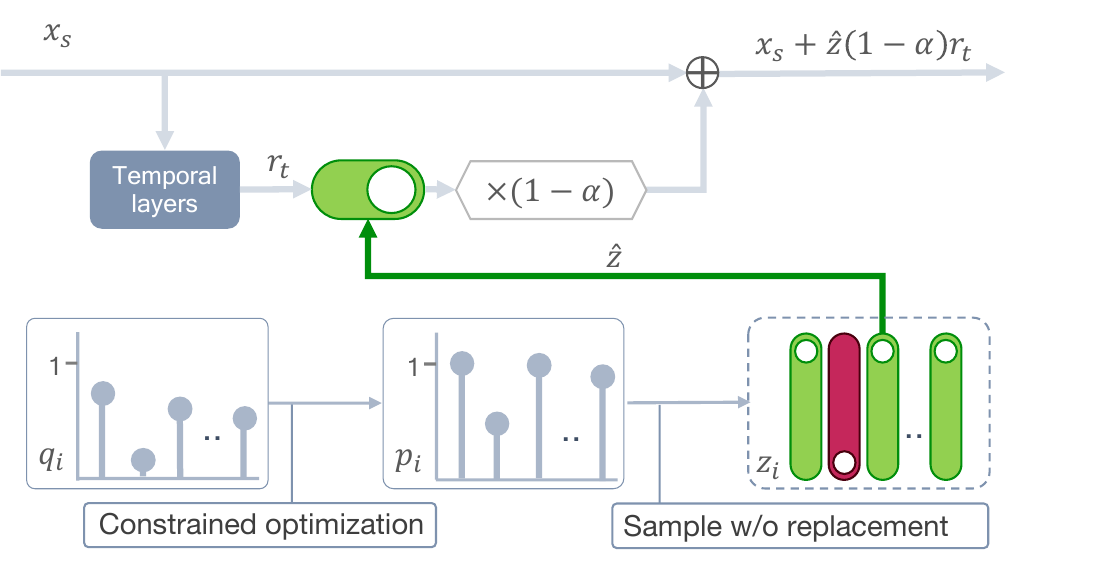}
\subcaption{
    A zero-one gate multiplier is sampled to each temporal block during training.
}
\label{fig:from_chances_to_gates_b}
\end{subfigure}
\caption{%
    \topic[0mm]{Learned pruning of temporal blocks.}
    (a) Each temporal block in the base SVD model is implemented as a residual block \wrt its input $x_s$. 
    The output of temporal layers $r_t$ is summed with the input $x_s$, and after that once again averaged with $x_s$ with learnable weight $\alpha$. 
    By reordering the terms, we derive the effective update rule $\alpha x_s +\left(1 - \alpha\right) x_t = x_s + \left(1 - \alpha\right) r_t $.
    (b) During training, we introduce a scalar gate $\hat{z} \in \left\lbrace 0, 1 \right\rbrace$ to the residual update rule of each block. We learn importance values $\left\{q_i\right\}_i$ of  temporal blocks which are transformed to inclusion probabilities $\left\{p_i\right\}_i$ at each training step. Zero-one gate multipliers are sampled according to those probabilities. To enable end-to-end training, we use straight-through estimator trick. 
    At inference, only $n$ blocks with highest importance values are used.
}
\label{fig:temporal_block}
\vspace{-2ex}
\end{figure}

%% file: figures/models_comparison/qualitative_comparison.tex
\begin{figure*}[t]
    \newlength{\mrgone}
    \setlength{\mrgone}{-0.6cm}
    \newlength{\wid}
    \setlength{\wid}{0.55\textwidth}
    \renewcommand{\cellalign}{cl}
    \centering
    \begin{tabular}{@{}cc@{}}
        \hspace{\mrgone} \includegraphics[width=\wid]{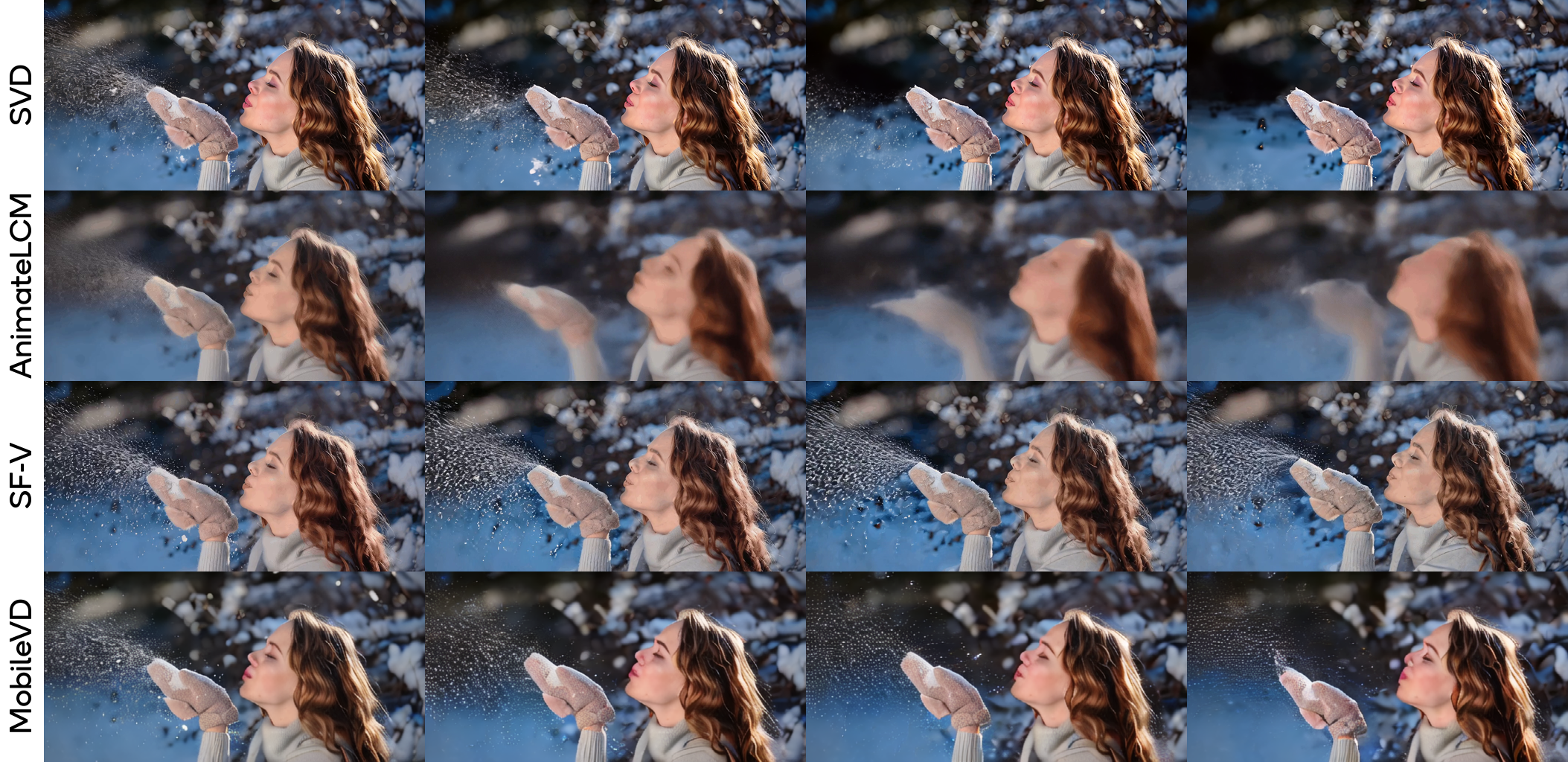}       
        & \hspace{\mrgone}\includegraphics[width=\wid]{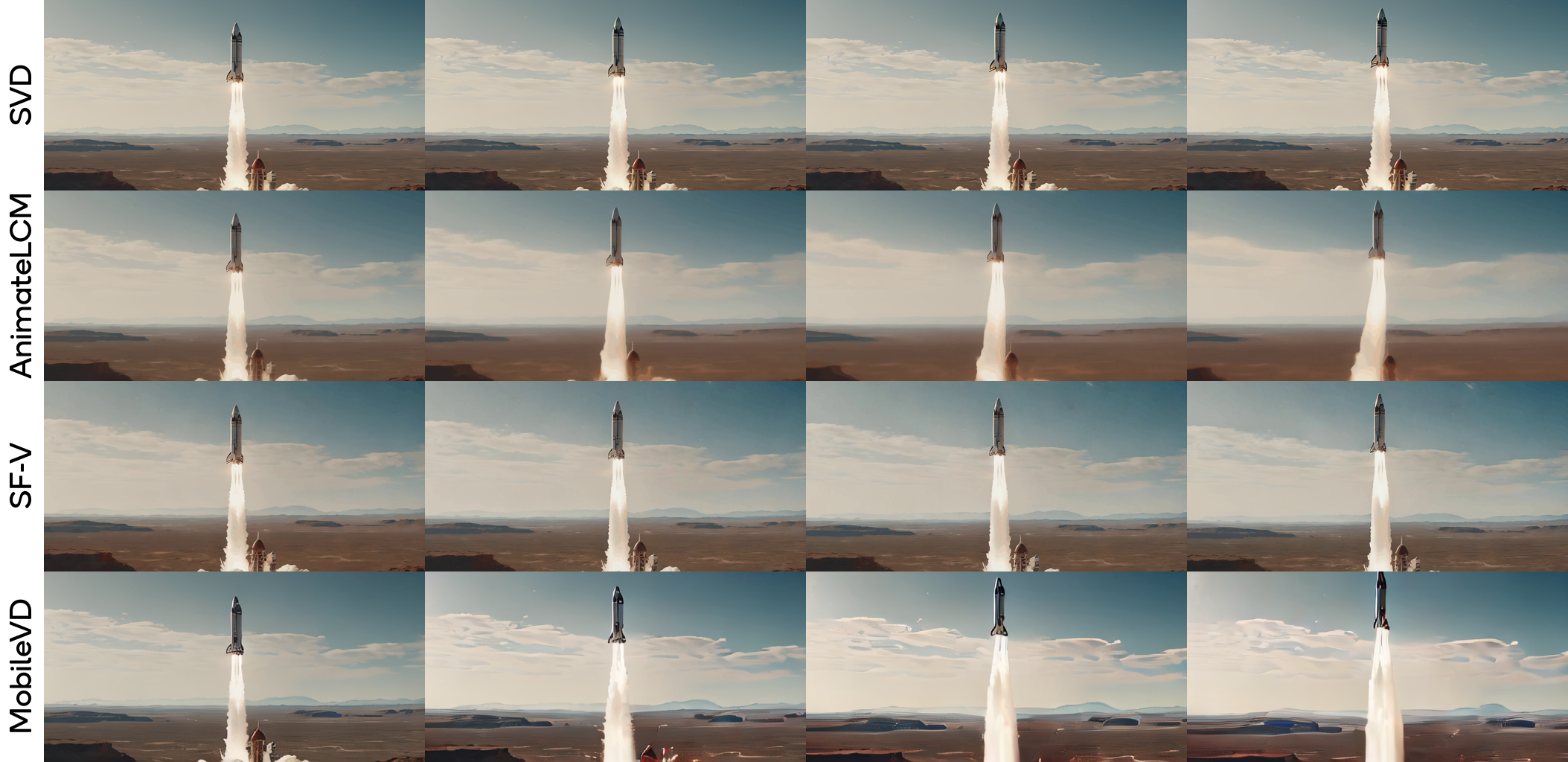} \\
        \hspace{\mrgone} \includegraphics[width=\wid]{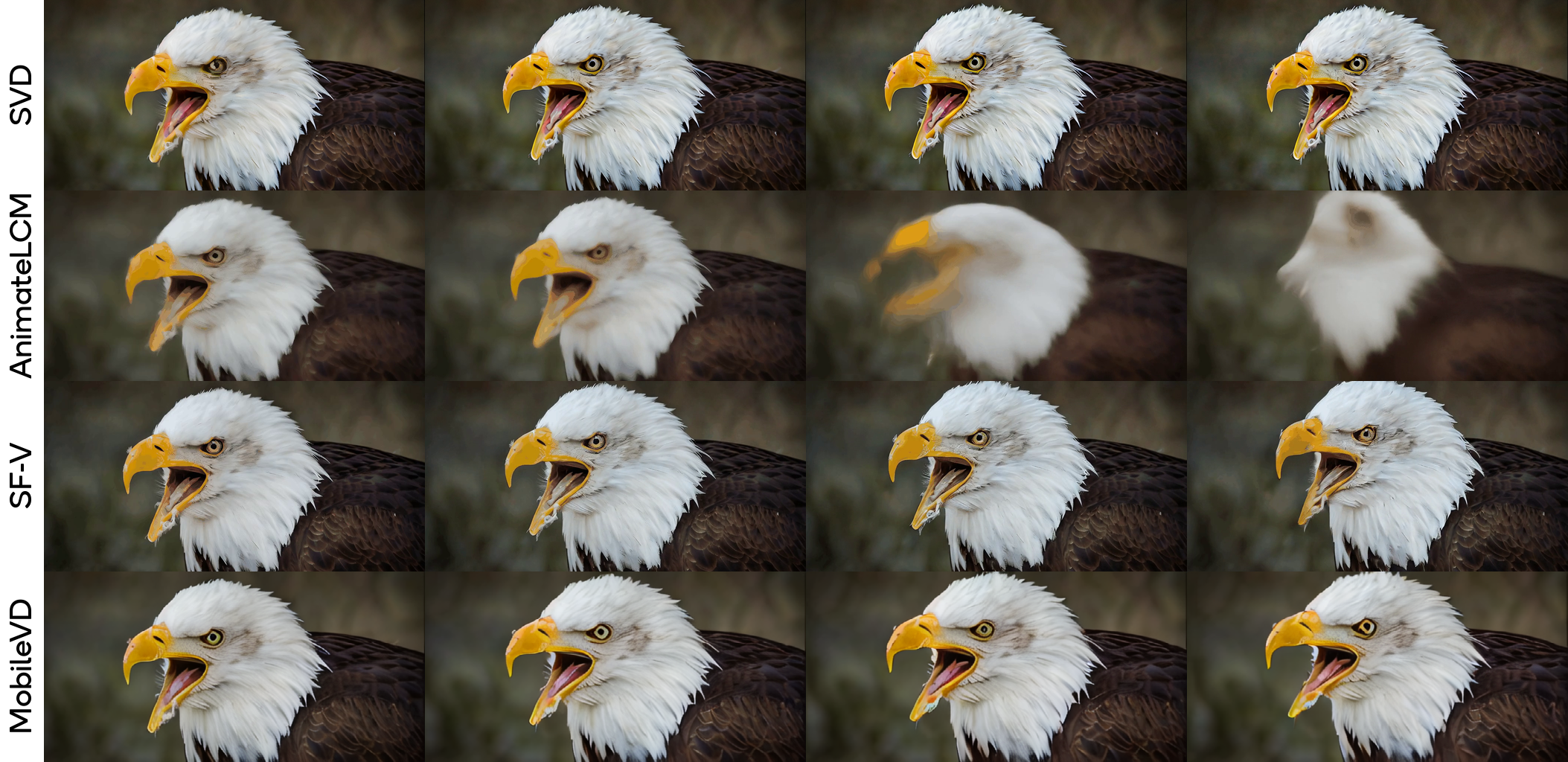}
        & \hspace{\mrgone} \includegraphics[width=\wid]{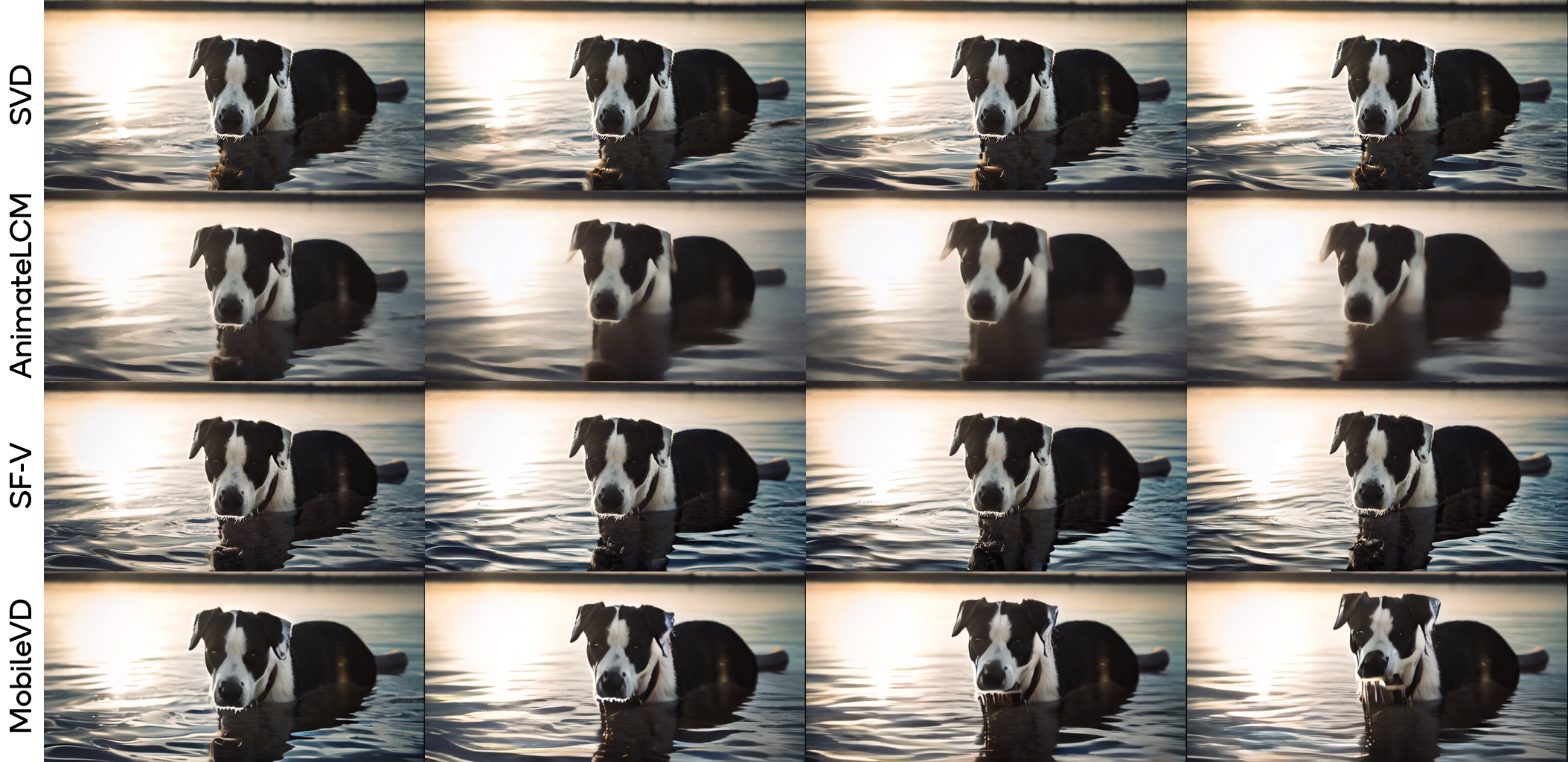} \\
        \hspace{\mrgone} \includegraphics[width=\wid]{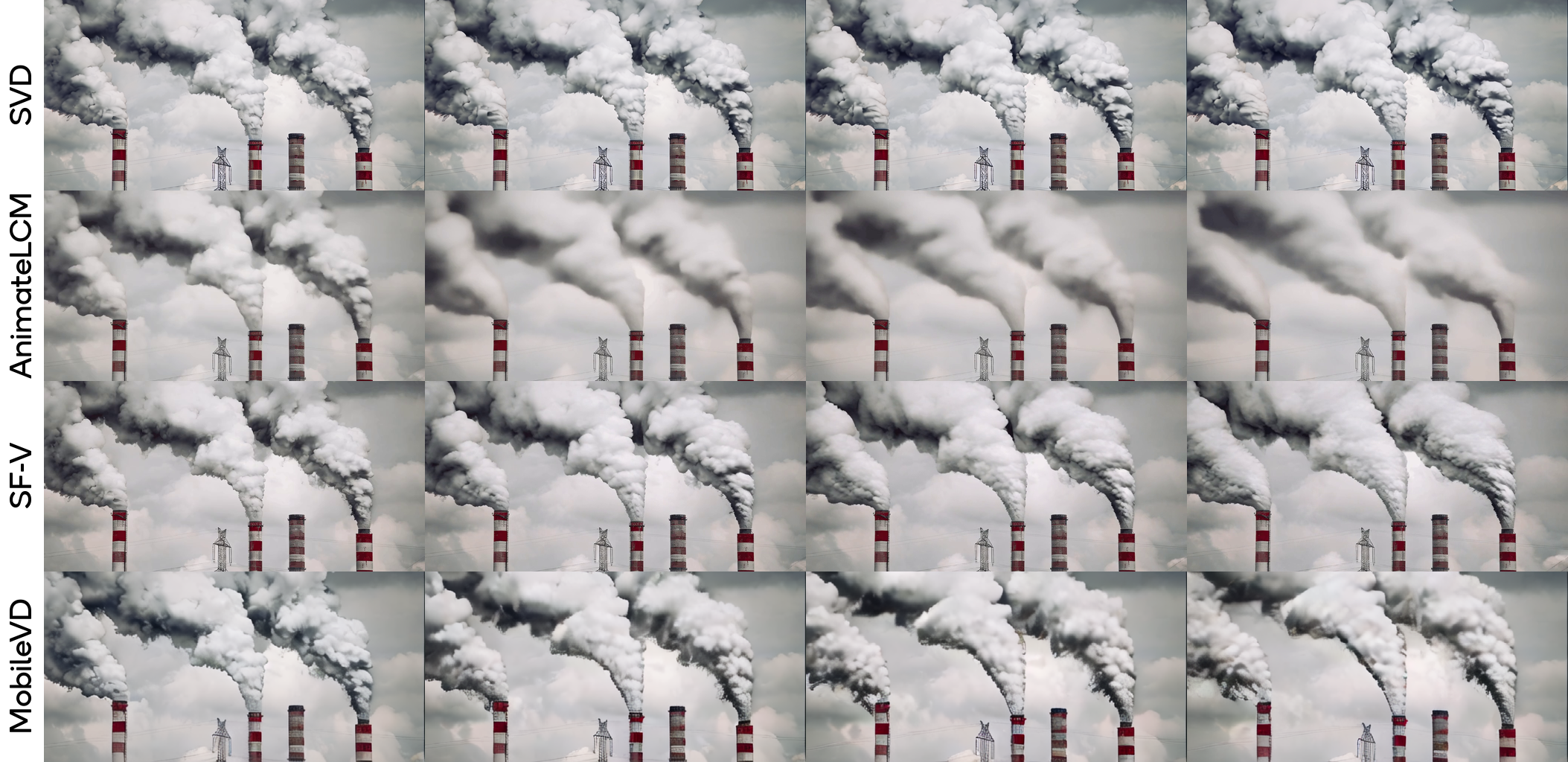}
        & \hspace{\mrgone} \includegraphics[width=\wid]{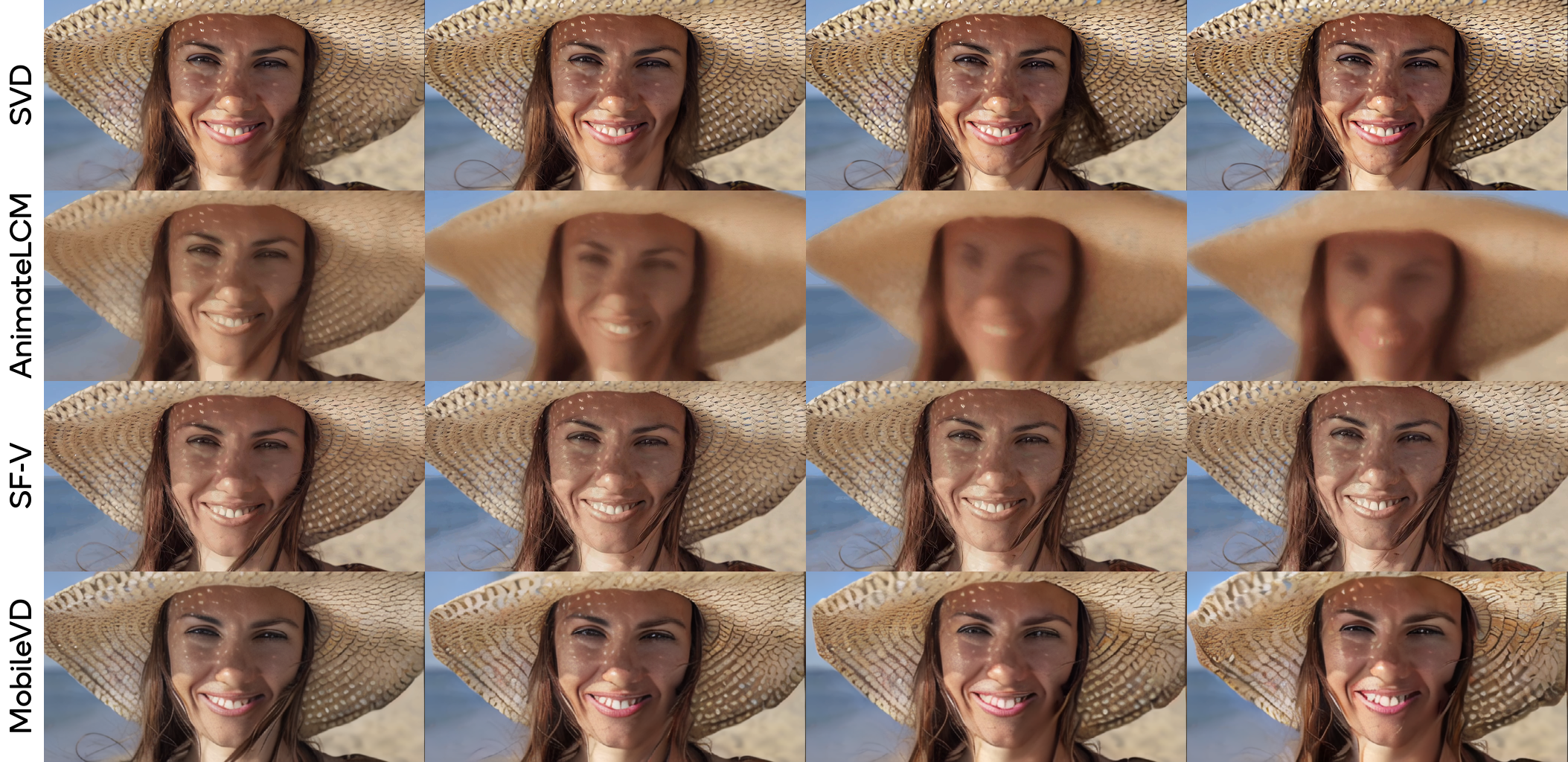}
    \end{tabular}

\caption{%
    \topic[0mm]{Comparison with recent models.}
    We provide the 1st, 6th, 10th and 14th frames from the videos generated with different models\protect\footnotemark.
    For AnimateLCM~\cite{wang2024animatelcm} and \sfv~\cite{zhang_svf_2024} we downsampled the released high-resolution videos from~\citet{zhang_svf_2024}.
    For SVD~\cite{blattmann_stable_2023} and our \methodname model, videos were generated at their native resolution, $1024 \times 576$ and $512 \times 256$ respectively.
    }
\label{fig:models_comparison}
\vspace{-1ex}
\end{figure*}

%% file: texts/4_experiments.tex
\section{Experiments}
\input{tables/sota_fvd}
In this section, we describe our experimental setup, followed by a qualitative and quantitative evaluation of our model. Finally, we present ablations to justify our choices for the final model deployed on the device.
\subsection{Implementation details}
\topic{Dataset.}
For our experiments, we use a collected video dataset. 
We follow the data curation pipeline from OpenSora~\cite[V.1.1]{opensora}, selecting videos with motion scores between \num{4} and \num{40}, and aesthetic scores of at least \num{4.2}. This results in a curated dataset of approximately 31k video files for finetuning.

\topic{Micro-conditioning.}
UNet used by SVD, has two conditioning parameters called \emph{FPS} and \emph{Motion bucket id}. 
To obtain videos with different FPS, we chose each $k$-th frame from the video with randomly sampled $k$, $1 \leq k \leq 4$, and adjusted the native FPS value of the video accordingly.
The notion of motion bucket has not been properly documented at time of  the model release.
While it is connected with the speed of motion in the video, as described in the original SVD paper, the motion estimator has not been open-sourced.
For that reason, we implemented our own definition of the motion bucket using a simple heuristic.
For the sampled chunk of 14 frames, we converted them to gray color, spatially downsampled to the resolution of $14 \times 128 \times 64$, and reshaped to the matrix of size $14 \times \left(128 \cdot 64\right)$. After that, we computed the singular values of that matrix. Note that for a completely static video this matrix has a rank of 1, and therefore the only non-zero singular value.
And the less similar frames are, the more singular components are needed to faithfully reconstruct the full video.
Based on that observation, we re-defined the motion bucket as the area under the normalized cumulative sum of singular values.

\topic{Training.}
For training, we begin with the original SVD weights, apply all the optimizations (excluding temporal block pruning), and train the resulting UNet with a standard diffusion loss for 100k iterations on 4 A100 GPUs with a total batch size of 8. This UNet serves as the initialization for adversarial finetuning, where a good initialization is crucial for fast convergence. We found that training for 5k iterations on 2 GPUs suffices for the second stage. 
We implement temporal block pruning in the second stage as we observed that excessive pruning in the first stage hinders model performance. In this case, we train the second stage for 10k steps. Check~\cref{supp:additional_details} for more details.

\topic{Metrics.}
We used DeepSpeed library~\cite[v0.14.2]{deepspeed_2020} to measure the number of FLOPs.
For GPU latency, NVIDIA\textsuperscript{\textregistered} A100 SXM\textsuperscript{\tiny{TM}} 80GB GPU was used.
To measure GPU latency, UNet model was compiled using the PyTorch~\cite[v2.0.1]{pytorch_neurips_2019} compiler with default settings. Phone latencies are measured on a Xiaomi-14 Pro that uses Snapdragon\textsuperscript{\textregistered} 8 Gen. 3 Mobile Platform with a Qualcomm\textsuperscript{\textregistered} Hexagon\textsuperscript{\tiny{TM}} processor. %
All performance metrics were measured for a single UNet evaluation with batch size of 1.
For video quality metric, we follow existing works~\cite{blattmann_stable_2023, zhang_svf_2024} by using Fréchet Video Distance (FVD)~\cite{unterthiner2019fvd} with I3D feature extractor~\cite{kay2017kineticshumanactionvideo}. We use the first frame of UCF-101 dataset~\cite{soomro2012ucf101dataset101human} as the conditioning image, generating 14-frame clips at the model’s native resolution.
Unless stated otherwise, we set the FPS, an SVD micro-condition, to \num{25}, matching the UCF-101 frame rate~\cite{blattmann_stable_2023}.
For motion bucket, for our models we used the median value at the specified frame rate across UCF-101 data.
For SVD the default bucket of 127 was used.

\subsection{Results}
\topic{\methodname.} 
In \cref{tab:optimizations_fvd} we compare \methodname to our base model. 
As the results indicate, each optimization reduces speed of inference on a mobile phone.
Decreased spatial resolution and optimized cross-attention together unlock on-device execution with a latency of 3.6 seconds. 
Temporal downsampling layers in UNet make inference 29\% faster.
Additionally, temporal blocks pruning  reduces phone latency by 13\%, and  channel funneling further decreases it by 9\%. 
Empirically, we found that a difference of up to  \num{20} FVD units does not significantly affect visual quality and typically falls within the standard deviation when using different random seeds.
Based on that, we see that our optimizations have minimal impact on FVD.

\topic{SOTA comparison.}
In \cref{tab:sota_fvd,fig:flops_vs_fvd_main} we compare to the recent works that similarly aim for accelerating SVD image-to-video model, namely, AnimateLCM~\cite{wang2024animatelcm}, LADD~\cite{sauer2024fast}, \sfv~\cite{zhang_svf_2024}, and UFOGen~\cite{Xu_2024_CVPR}. We observe that \methodname leads to a better FVD with a significantly less computation.

\topic{Qualitative results.}
Following previous works~\cite{wang2024animatelcm, zhang_svf_2024}, we show qualitative results with a commonly used set of conditioning images. Sampled frames from the generated videos are presented in~\cref{fig:models_comparison}.
For this visualization, we generated videos at 7 FPS and with spatial resolution of $512 \times 256$ using our \methodname.
Please refer to our webpage for the full videos. 
We observe that in general our method produces videos with sharp details and consistent motion.

\subsection{Ablations}\label{subsec:ablations}
In this section, we evaluate our design choices through ablation experiments. Unless otherwise specified, we use the SVD checkpoint with low-resolution input, optimized cross-attention, and adversarial finetuning as the reference model, \cf~\cref{tab:optimizations_fvd}.

\input{tables/lightweight_ablation}
\subsubsection{Resolution impact in UNet}\label{subsubsec:reslution_unet}
In~\cref{tab:lightweight_ablation} we compare different latent multiscaling optimizations proposed in~\cref{sec:lightweight_unet}. Specifically, we investigate the impact of inserting spatial or temporal multiscaling layers after the first UNet block in terms of FLOPs, latency, and FVD.
Spatial multiscaling offers better FLOPs and latency than temporal multiscaling and it increases FVD by 12 units compared to 5 for temporal downsampling. While we typically do not see video degradation with this increase in FVD, we do see clear degradation in video quality when using spatial instead of temporal multiscaling. We hypothesize that this is because the model already enjoys multiple stages of spatial downsampling, while temporal downsampling was originally absent. Based on these results, we have opted for temporal downsampling for our mobile-deployed model. We hold similar conclusions for combining the two multiscaling approaches with spatiotemporal multiscaling.

\subsubsection{Funnel finetuning}
\topic{Fun-factor and funnel initialization.}
Reducing the width of affine layers in the model is a form of lossy compression, and overly aggressive fun-factor values will hurt the model performance. In \cref{tab:init_fun_factor}, we observe the impact of the fun-factor. Reducing the fun-factor to $0.25$ results in a performance loss of 22 FVD units compared to fun-factor of $1$ (\ie, no compression). To avoid performance degradation from stacking multiple optimizations described in \cref{sec:method}, we set the fun-factor to $0.5$ for the deployed model. Additionally, the results highlight that the proposed coupled singular initialization (CSI) is essential for optimal funnel behavior, whereas the standard He initialization~\cite{he2015delvingDI} is suboptimal.

\input{tables/fun_factor_and_init}

\topic{Funnel merging and low-rank layers.}
We compare channel funnels with multiple baselines in terms of FLOPs, on-device latency and FVD. All baselines are applied on the same attention layers, unless specified otherwise. The first baseline uses channel funnels but merges funnel and weight matrices during training, hence mimicking behavior at inference time as shown in~\cref{fig:funnel_pipeline}.
We report in~\cref{tab:funnel_vs_svd_baseline} that keeping funnel and weight matrices separate at training performs equally well. The second baseline is applying funnels to convolutions in ResNet blocks instead of attention layers. While we obtain favourable FVD and even greater gain in TFLOPs (7.2 vs 8.6), it does not translate to the latency reduction we see with funnels on attention (3.40 vs 2.87 seconds). We hypothesize that the attention layers play a greater role in reducing latency on device than convolutions for this model. The last baselines employ the standard technique of truncated singular decomposition of individual layers~\cite{zhang2015acceleratingVD}. This decomposition breaks down a weight matrix of a linear layer $W \in \mathbb{R}^{\cout \times \cin}$ into a low-rank product of two matrices $W_1 \in \mathbb{R}^{rc  \times \cin}$ and $W_2 \in \mathbb{R}^{\cout \times rc}$ where $r$ is the rank reduction factor, $r < 1$, and $c = \min\left(\cin, \cout\right)$.
Note that reduction in the number of parameters and FLOPs is achieved only if $r < 0.5$, while the size of the feature map after these two matrices remain intact in this approach.
While truncated decomposition after finetuning performs well in terms of FVD for both $r=0.25$ and $r=0.5$, it is slower on device compared to channel funneling (\num{3.35} and \num{3.48} vs. \num{2.87} seconds respectively). This difference is attributed to memory transfer overhead from introducing additional layers as well as not decreasing the original feature size, emphasizing the benefit of funnels. 
 
\input{tables/funel_vs_svd_baseline}

\subsubsection{Temporal blocks pruning}
\input{tables/temp_block_ablation}
As mentioned in~\cref{sec:temporal_adaptors}, in our experiments we found it possible to reduce up to 70\% of all temporal blocks in the UNet.
Notably, even with such a high pruning rate the quality is comparable to the original model: it achieves FVD of \num{127} and requires 14\% less FLOPs, while the original model has FVD of 139.
However, pruning even further seems far from straightforward.
FVD degrades to the values above 200 with a pruning rate of 80\%, and visual quality drops drastically.

Additionally, we compare our method with another pruning baseline. As described in~\cref{sec:temporal_adaptors}, the output of a spatiotemporal block in SVD is a linear combination of the spatial and temporal blocks $x_s$ and $x_t$, respectively, $\alpha x_s + (1-\alpha) x_t$, where $\alpha$ is a weight scalar. The baseline aims to minimize the influence of temporal blocks, by minimizing $1-\alpha$. This is done by adding an additional loss term during finetuning: $\mathcal{L} =  \lambda \sum\nolimits_i \left(1 - \alpha_i\right)$.
After training, blocks with the highest weight scalar $\alpha_i$ are pruned.
However, this method lacks explicit control over the desired pruning rate, as only weight hyperparameter $\lambda$ can be adjusted. While effective for small pruning rates, this approach did not allow us to remove as many blocks as our method, achieving only a 7\% reduction in FLOPs with acceptable FVD level, see~\cref{tab:temporal_pruning_ablation}.

%% file: tables/sota_fvd.tex
\begin{table}[t]
\newcommand{\leftspace}{}
\centering
\resizebox{.99\columnwidth}{!}{
\def\arraystretch{1.}%
\begin{tabular}{lccccc} 
\toprule
\leftspace \multirow{2}{*}{Model} %
& \multirow{2}{*}{NFE} %
& \multirow{2}{*}{FVD $\downarrow$} %
& \multirow{2}{*}{TFLOPs $\downarrow$} %
& \multicolumn{2}{c}{Latency (ms) $\downarrow$} \\ 
\cmidrule{5-6}
& &  &  & GPU & Phone \\
\midrule
\multicolumn{6}{l}{\cellcolor{gray!20}{\emph{Resolution} $\mathit{1024 \times 576}$}} \\ %
\leftspace SVD         & 50 & $149$   & 45.43 & 376 & OOM \\
\leftspace AnimateLCM$^*$  & 8 & $281$   & 45.43 & 376 & OOM \\
\leftspace UFOGen$^*$    & 1  & $1917$ & 45.43 & 376 & OOM \\
\leftspace LADD$^*$       & 1  & $1894$ & 45.43 & 376 & OOM \\
\leftspace \sfv{}$^*$        & 1  &  $181$ & 45.43 & 376 & OOM \\ 
\multicolumn{6}{l}{\cellcolor{gray!20}{\emph{Resolution} $\mathit{512 \times 256}$}} \\ %
\leftspace SVD         & 50 & 476   & 8.60 & 82 & OOM \\
\leftspace \methodname (ours) & 1 &   171 & 4.34 & 45 & 1780 \\
\bottomrule
\end{tabular}
}
\caption{
\topic[0mm]{Comparison with recent models.} FLOPs and latency are provided for a single function evaluation with batch size of 1. For rows marked with asterisk$^*$ FVD measurements were taken from~\citet{zhang_svf_2024}, while performance metrics are based on our measurements for UNet used by SVD. 
For consistency with these results, FVD for SVD and our \methodname model was measured on UCF-101 dataset at 7 frames per second. 
}
\label{tab:sota_fvd}
\vspace{-2ex}
\end{table}

\begin{table*}[t]
\newcommand{\leftspace}{\phantom{a}}
\centering
\def\arraystretch{1.1}%
\setlength{\tabcolsep}{12pt} %
\resizebox{\textwidth}{!}{%
\footnotesize  %
\begin{tabular}{lcccccc} %
\toprule
\multirow{2}{*}{Model} %
& \multirow{2}{*}{NFE} %
& \multicolumn{2}{c}{FVD $\downarrow$} %
& \multirow{2}{*}{TFLOPs $\downarrow$} %
& \multicolumn{2}{c}{Latency (ms) $\downarrow$} \\ 
\cmidrule{3-4}\cmidrule{6-7}
& & 25 FPS & 7 FPS &  & GPU & Phone \\
\midrule
SVD (resolution $1024 \times 576$)        & 50 & 140 & $149$   & 45.43 & 376 & OOM \\
SVD (resolution $512 \times 256$)         & 50  & 366 & 476   & 8.60 & 82 & OOM \\
\leftspace + low-resolution finetuning & 50 & 194 & 196 & 8.60 & 82 & OOM \\
\leftspace\leftspace + optimized cross-attention & 50 & 194 & 196 & 8.24 & 76 & 3630 \\
\leftspace\leftspace\leftspace + adversarial finetuning & 1 &  133 & 168 & 8.24 & $76$ & 3630 \\
\leftspace\leftspace\leftspace\leftspace + temporal multiscaling & 1  &  139 & 156 & 5.42 & $59$ & 2590\\
\leftspace\leftspace\leftspace\leftspace\leftspace + temporal block pruning & 1  &  127 & 150 & 4.64 & 47 &  2100 \\
\leftspace\leftspace\leftspace\leftspace\leftspace\leftspace + channel funneling & 1  &   149 & 171 & 4.34 & 45 & 1780 \\
\bottomrule
\end{tabular}
}
\caption{
    \topic[0mm]{Effect of our optimizations.} 
    We successfully deployed the image-to-video model to a mobile device without significantly sacrificing the visual quality.
    FLOPs and latency are provided for a single function evaluation with batch size of 1.
    We call the model in the bottom row Mobile Video Diffusion, or \methodname.
}
\label{tab:optimizations_fvd}
\vspace{-3ex}
\end{table*}

%% file: tables/lightweight_ablation.tex
\begin{table}[t]
\newcommand{\leftspace}{\phantom{a}}
\centering
\resizebox{0.99\columnwidth}{!}{
\def\arraystretch{1.2}%
\begin{tabular}{cccccc} 
\toprule
\multirow{2}{*}[-0.5mm]{\makecell{Spatial\\multiscaling}}
& \multirow{2}{*}[-0.5mm]{\makecell{Temporal\\multiscaling}}
& \multirow{2}{*}{FVD $\downarrow$} 
& \multirow{2}{*}{TFLOPs $\downarrow$}
& \multicolumn{2}{c}{Latency (ms) $\downarrow$}
\\%
\cmidrule{5-6}
 &  & & & GPU & Phone\\
\midrule
 $\times$ & $\times$ & 133 & 8.24 & 76 & 3630 \\ %
 $\times$ & \checkmark  & 138 & 5.42 & 59 & 2590 \\ %
 \checkmark & $\times$ & 145 & 4.35 & 51 & 2280 \\ %
 \checkmark & \checkmark & 163 & 3.39 & 48 & --- \\ %
\bottomrule
\end{tabular}
}
\caption{
\topic[0mm]{Effect of additional multiscaling layers in UNet.}
We observe that both temporal and spatial multiscaling has good impact on mobile latency without compromising much on FVD, while combining the two increases FVD by a noticeable amount.
}\label{tab:lightweight_ablation}
\vspace{-3ex}
\end{table}

%% file: tables/fun_factor_and_init.tex
\begin{table}[t]
\centering
\resizebox{\columnwidth}{!}{
\def\arraystretch{1.2}%
\newcommand{\leftspace}{}
\begin{tabularx}{1.2\linewidth}{Xlcc}
\toprule
\leftspace Initialization & \makecell{Fun-factor} & FVD $\downarrow$ \\ 
\midrule
Coupled singular init. (CSI)          & \leftspace 0.25  & 155      \\
Coupled singular init. (CSI)          & \leftspace 0.50  & 132      \\
Coupled singular init. (CSI)          & \leftspace 0.75  & 145      \\
Coupled singular init. (CSI)          & \leftspace 1.00  & 133      \\
He init.~\cite{he2015delvingDI}        & \leftspace 0.50  & 332      \\
\bottomrule
\end{tabularx}
}
\vspace{-1ex}
\caption{
\topic[0mm]{Effect of funnel initialization and fun-factor.}
Initialization funnels with CSI is crucial to getting good FVD as He initialization~\cite{he2015delvingDI} obtains roughly 200 FVD units more. Additionally, we see that reducing the fun-factor beyond 0.5 starts to affect the performance.
} 
\label{tab:init_fun_factor}
\vspace{-2ex}
\end{table}

%% file: tables/funel_vs_svd_baseline.tex
\begin{table}[t]
\centering
\newcommand{\leftspace}{\phantom{a}}
\resizebox{0.99\columnwidth}{!}{
\def\arraystretch{1.2}%
\begin{tabular}{lccccc}
\toprule%
Width reduction method & $r$ & FVD $\downarrow$ & TFLOPs $\downarrow$ & Latency (ms) $\downarrow$ \\
\midrule
Original UNet                                  & - & 133  & 8.6  & 3630 \\
\rowcolor{gray!20}
\leftspace + Funnels                           & 0.5   & 132  & 8.0  & 2870     \\
\leftspace + Funnels (merge before finetune)   & 0.5   & 138  & 8.0  & 2870     \\
\leftspace + Funnels (convolutions)            & 0.5   & 139  & 7.2  & 3400     \\
\leftspace + Truncated singular decomposition  & 0.5   & 142  & 8.6  & 3482     \\
\leftspace + Truncated singular decomposition  & 0.25   & 130  & 8.0 & 3345     \\
\bottomrule
\end{tabular}
}
\vspace{-1ex}
\caption{
\topic{Comparison of model width reduction methods.}
We compare the proposed channel funneling (in grey) with finetuned low-rank approximation of individual attention layers with truncated singular decomposition. We additionally compare to Funnels applied to convolutions instead of attention. The reduction rate (referred to as fun-factor in case of funnels) is highlighted with $r$.
} 
\label{tab:funnel_vs_svd_baseline}
\vspace{-2ex}
\end{table}

%% file: tables/temp_block_ablation.tex
\begin{table}[t]
\centering
\resizebox{0.95\columnwidth}{!}{
\def\arraystretch{1.0}%
\newcommand{\leftspace}{}
\begin{tabularx}{\linewidth}{lcccc} 
\toprule
\leftspace Blocks pruned (\%) & FVD $\downarrow$ & TFLOPs $\downarrow$ & \makecell{Latency \\GPU (ms) $\downarrow$} \\
\midrule
\hline
\rowcolor{gray!20}
\multicolumn{3}{l}{\emph{Our method}} &\\
\leftspace 90 & 201 & 4.06 & 42 \\ %
\leftspace 80 & 245 & 4.35 & 44 \\ %
\leftspace 70 & 127 & 4.64 & 47 \\ %
\rowcolor{gray!20}
\multicolumn{3}{l}{\emph{$L_1$ regularization}} &\\
\leftspace 70 & 207 &  4.67 & 48\\ %
\leftspace 53 & 165 & 5.17 & 52 \\ %
\bottomrule
\end{tabularx}
}
\vspace{-1ex}
\caption{
    \topic[0mm]{Impact of temporal blocks pruning.}
    Our pruning outperforms the $L_1-$ regularization which does not have explicit control over the number of removed blocks. We use the checkpoint, optimized up to the temporal block pruning stage, as the starting point. 
}
\label{tab:temporal_pruning_ablation}
\vspace{-2ex}
\end{table}

%% file: texts/5_conclusion.tex
\section{Conclusion}
\label{sec:conclusion}
This paper introduced the first mobile-optimized video diffusion model, addressing the high computational demands that have limited their use on mobile devices. By optimizing the spatio-temporal UNet from Stable Video Diffusion and employing novel pruning techniques, we significantly reduced memory and computational requirements. Our model, \methodname, achieves substantial efficiency improvements with minimal quality loss, making video diffusion technology feasible for mobile platforms.

\topic{Limitations.}
Despite the impressive acceleration achieved as a steppingstone towards video generation on phones, the output is currently limited to 14 frames at a low resolution of $256\times512$ pixels. The next step involves leveraging more efficient autoencoders to achieve higher spatial and temporal compression rates, enabling the generation of larger and longer videos at the same diffusion latent generation cost.

\topic{Potential negative impacts.}
Our work is a step towards making video generation technology accessible to a broader audience, allowing users to create content on their phones with fewer access controls and regulations compared to cloud-based solutions.

%% file: supplementary_chapters/0_results.tex
\input{tables/sota_fvd_ext}
\section{Additional results}

\subsection{High-resolution generation}
In addition to our main effort to port the SVD model on a mobile phone, we tested if our set of optimizations can be applied to a high-resolution model.
For this purpose, we trained a model called \methodnamehires which is capable of generating 14-frame videos with spatial size of $1024 \times 576$ px.
Architecture hyperparameters used to finetune \methodnamehires, \ie temporal multiscaling factor, fun-factor and number of pruned temporal blocks, were the same as for our low-resolution \methodname.
We report the results of that model in~\cref{tab:sota_fvd_ext}.
As shown, it achieves visual quality on par with \sfv while decreasing its GPU latency by \num{40}\%.

\subsection{Visual quality metrics}
Although FVD is a widely used metric for visual quality in the community, its ability to assess the faithfulness of motion as opposed to the appearance of individual frames is sometimes argued~\cite{Ge_2024_CVPR,brooks2022generating,Skorokhodov_2022_CVPR}.
For that reason, we also evaluate the quality of different models using a recently proposed \jedi (JEPA Embedding Distance) metric~\cite[v.0.1.3]{luo2024jedi}.
\jedi reportedly demonstrates much better correlation with the human preference than FVD.
Since \citet{luo2024jedi} have not recommended guidelines for comparison between generative models using their metric, we opted for the setup similar to FVD computation~\cite{blattmann_stable_2023,zhang_svf_2024}.
In detail, we used the same set of clips from UCF-101, and similarly downsampled the videos to the resolution closest to $320 \times 240$, preserving the original aspect ratio, with the central crop afterwards. 
In~\cref{tab:optimizations_fvd_ext} we provide the extended quantitative results.
We observe that in general the new metric is better aligned with architecture optimizations that we apply. 

%% file: tables/sota_fvd_ext.tex
\begin{table}[t]
\newcommand{\leftspace}{}
\centering
\resizebox{.99\columnwidth}{!}{
\def\arraystretch{1.}%
\begin{tabular}{lccccc} 
\toprule
\leftspace \multirow{2}{*}{Model} %
& \multirow{2}{*}{NFE} %
& \multirow{2}{*}{FVD $\downarrow$} %
& \multirow{2}{*}{TFLOPs $\downarrow$} %
& \multicolumn{2}{c}{Latency (ms) $\downarrow$} \\ 
\cmidrule{5-6}
& &  &  & GPU & Phone \\
\midrule
\multicolumn{6}{l}{\cellcolor{gray!20}{\emph{Resolution} $\mathit{1024 \times 576}$}} \\ 
\leftspace SVD         & 50 & $149$   & 45.43 & 376 & OOM \\
\leftspace AnimateLCM$^*$  & 8 & $281$   & 45.43 & 376 & OOM \\
\leftspace UFOGen$^*$    & 1  & $1917$ & 45.43 & 376 & OOM \\
\leftspace LADD$^*$       & 1  & $1894$ & 45.43 & 376 & OOM \\
\leftspace \sfv{}$^*$        & 1  &  $181$ & 45.43 & 376 & OOM \\ 
\leftspace \methodnamehires (ours) & 1  & $184$ & 23.63 & 227 & OOM \\ 
\multicolumn{6}{l}{\cellcolor{gray!20}{\emph{Resolution} $\mathit{512 \times 256}$}} \\ 
\leftspace SVD         & 50 & 476   & 8.60 & 82 & OOM \\
\leftspace \methodname (ours) & 1 &   171 & 4.34 & 45 & 1780 \\
\bottomrule
\end{tabular}
}
\caption{
\topic[0mm]{Comparison with recent models.} The set of optimizations proposed in our paper, can also be applied to high-resolution generation. FLOPs and latency are provided for a single function evaluation with batch size of 1. For rows marked with asterisk$^*$ FVD measurements were taken from~\citet{zhang_svf_2024}, while performance metrics are based on our measurements for UNet used by SVD. 
For consistency with these results, FVD for SVD and our  models was measured on UCF-101 dataset at 7 frames per second. 
}
\label{tab:sota_fvd_ext}
\end{table}

\begin{table*}[t]
\newcommand{\leftspace}{\phantom{a}}
\centering
\def\arraystretch{1}%
\begin{tabular}{lrrrcrrrrc} %
\toprule
\multirow{2}{*}{Model} %
& \multirow{2}{*}{NFE} %
& \multicolumn{2}{c}{FVD $\downarrow$} %
& %
& \multicolumn{2}{c}{\jedi{} $\downarrow$} %
& \multirow{2}{*}{TFLOPs $\downarrow$} %
& \multicolumn{2}{c}{Latency (ms) $\downarrow$} \\ 
\cmidrule{3-4}\cmidrule{6-7}\cmidrule{9-10}
& & 25 FPS & 7 FPS & & 25 FPS & 7 FPS & & GPU & Phone \\
\midrule
\multicolumn{10}{l}{\cellcolor{gray!20}{\emph{Resolution} $\mathit{1024 \times 576}$}} \\ 
SVD         & 50 & 140 & $149$  & & 0.61 & 0.59 & 45.43 & 376 & OOM \\
\methodnamehires & 1 & 126 & 184 & & 0.96 & 1.75 & 23.63 & 227 & OOM \\
\multicolumn{10}{l}{\cellcolor{gray!20}{\emph{Resolution} $\mathit{512 \times 256}$}} \\ 
SVD          & 50  & 366 & 476 & & 1.05 & 1.14 & 8.60 & 82 & OOM \\
\leftspace + low-resolution finetuning & 50 & 194 & 196 & & 0.71 & 0.65 & 8.60 & 82 & OOM \\
\leftspace\leftspace + optimized cross-attention & 50 & 194 & 196 & & 0.71 & 0.65 & 8.24 & 76 & 3630 \\
\leftspace\leftspace\leftspace + adversarial finetuning & 1 &  133 & 168 & & 0.66 & 0.71 & 8.24 & $76$ & 3630 \\
\leftspace\leftspace\leftspace\leftspace + temporal multiscaling & 1  &  139 & 156 & & 0.83 & 0.81 & 5.42 & $59$ & 2590\\
\leftspace\leftspace\leftspace\leftspace\leftspace + temporal block pruning & 1  &  127 & 150 & & 0.97 & 1.32 & 4.64 & 47 &  2100 \\
\leftspace\leftspace\leftspace\leftspace\leftspace\leftspace + channel funneling & 1  &   149 & 171 & & 1.07 & 1.21 & 4.34 & 45 & 1780 \\
\bottomrule
\end{tabular}
\caption{
    \topic[0mm]{Effect of our optimizations.} 
    We successfully deployed the image-to-video model to a mobile device without significantly sacrificing the visual quality.
    FLOPs and latency are provided for a single function evaluation with batch size of 1.
    We call the model in the bottom row Mobile Video Diffusion, or \methodname.
    The model trained with the same hyperparameters but intended for high-resolution generations is referred to as \methodnamehires.
}
\label{tab:optimizations_fvd_ext}
\end{table*}

%% file: supplementary_chapters/1_method_details.tex
\section{Additional details}
\label{supp:additional_details}
\subsection{Training}
For diffusion training, we used AdamW optimizer~\cite{loshchilov2018decoupled} with learning rate of \num{1e-6} and weight decay \num{1e-3}, while other hyperparameters were default.
During adversarial finetuning, we also used AdamW.
For generator the learning rate was equal to \num{1.25e-6}, and for discriminator we set it 10 times higher.
For logits of importance values, used for pruning of temporal blocks, learning rate was equal to \num{1e-3}.
Momentum weights for both optimizers we set as follows $\beta_1=0.5, \, \beta_2=0.999.$
For generator, the weights for adversarial and pseudo-Huber loss were equal to 1 and 0.1 respectively. 
For discriminator, weight of $R_1$ penalty was \num{1e-6}, and we applied it once per 5 iterations, as recommended in~\cite{Karras_2020_CVPR}.

For high-resolution training, the first (diffusion) stage lasted for 10k iterations with total batch size of 4.
The second (adversarial) stage comprised 30k iterations with batch size of 2. The learning rates for adversarial finetuning were twice as lower as for low-resolution case, except for the learning rate of importance values. 
$R_1$ penalty was applied at each step. 
Other training aspects were the same both for \methodname and \methodnamehires.

\subsection{Decoding latents}
As our approach stems from SVD, the presented \methodname is also a latent model.
Therefore, each generated latent code must be decoded to raw RGB pixels.
For all the experiments reported in the main text, we used the native autoencoder weights released alongside the SVD model itself.
The decoding was conducted independently for each frame.
Notably, this model is relatively slow on device: it takes \num{91.5} ms per frame, resulting in \num{1280} ms for decoding the full generated video.
This timing is comparable with the latency of \methodname (\num{1780} ms). 
As an alternative, we used the decoder from TAESD autoencoder\footnote{\url{https://huggingface.co/madebyollin/taesd/tree/614f768}}.
It is significantly faster on a smartphone: \num{6.4} ms per frame, or \num{90} ms for the full video.
At the same time, the difference in quality metrics is negligible, see~\cref{tab:tiny_ae}.

\input{tables/tiny_ae}

\subsection{Channel funnels}
\topic{Attention layers.}
Consider a query and key projection matrices in a self-attention similarity map computation, $X W_q \left( X W_k \right)^T$ with layer input $X$ having a shape of $L \times \cin$, and $W_q$ and $W_k$ of $\cin \times \cinner$.
With funnel matrices $F_q$ and $F_k$ of size $\cinner \times c'$, we modify the aforementioned bilinear map as $X W_q F_q \left( X W_k F_k \right)^T = X W_q F_q F_k^T W_k^T X^T$. 
We apply our coupled singular initialization (CSI) by setting  $F_q = W_q^\dagger U_{c'} \Sigma_{c'}^{1/2}$ and $F_k = W_k^\dagger V_{c'} \Sigma_{c'}^{1/2}$.
For value and output projections matrices the initialization is applied in the way discussed in the main text.

\topic{Convolutional layers.}
Applying the same  initialization to a pair of convolutional layers is not straightforward. 
Weight tensors of 2D convolutional layers are 4-dimensional, and therefore computation of the effective weight tensor is not obvious. 
However, we make use of the following observation.
Consider input tensor $X$ with shape $h \times w \times \cin$. 
We refer to its 2-dimensional pixel coordinate as $\vec{p}$, and therefore $X_{\vec{p}}$ is a vector with $\cin$ entries.
Let $W$ be a convolutional kernel with size $k_h \times k_w \times \cout \times \cin$, and we refer to its spatial 2-dimensional coordinate as $\vec{q}$, while $\vec{q} = 0$ is a center of a convolutional filter.
For $j$-th output channel, $W_{\vec{q},j}$ is also a vector with $\cin$ entries.
The layer output $Y \in \mathbb{R}^{h \times w \times \cout}$ can be computed as 
\begin{equation}
    Y_{\vec{p},j} = \sum\nolimits_{\vec{q}} \left\langle W_{\vec{q},j}, \, X_{\vec{p} + \vec{q}} \right\rangle,
    \label{eq:conv_input_patch}
\end{equation}
where $ \left\langle \cdot , \, \cdot \right\rangle$ denotes inner product.
Simply speaking, this means that convolution can be treated as a linear layer with weight matrix of shape $\cout \times \left( k_h \cdot k_w \cdot \cin \right)$ applied to each flattened input patch.

At the same time, another way of reshaping the kernel is also possible.
Consider a tensor $E$ of shape $k_h \times k_w \times \cout \times h \times w$ defined as
\begin{equation}
    E_{\vec{q},j,\vec{p}} = \left\langle W_{\vec{q},j}, \, X_{\vec{p}} \right\rangle.
    \label{eq:conv_output_collection}
\end{equation}
In other words, the convolution kernel reshaped as $\left( k_h \cdot k_w \cdot \cout \right) \times \cin$ is multiplied by features of each input pixel.
Then \cref{eq:conv_input_patch} can be rewritten as 
\begin{equation}
    Y_{\vec{p},j} 
    = \sum\nolimits_{\vec{q}} E_{\vec{q},j,\vec{p}+\vec{q}}
    = \sum\nolimits_{\vec{q}} \left\langle E_{\vec{q},j}, \, \delta_{\vec{p}+\vec{q}} \right\rangle,
    \label{eq:conv_reduction}
\end{equation}
where $\delta$ is a 4-dimensional identity tensor, \ie $\delta_{\vec{u},\vec{v}} = 1$ if $\vec{u} =\vec{v}$ and 0 otherwise.

Having said that, a sequence of two convolutions can be presented as 
\begin{enumerate*}[label=(\roman*)]
\item flattening the input patches; 
\item matrix multiplication by the first kernel reshaped according to~\cref{eq:conv_input_patch}; followed by 
\label{listitem:1}
\item matrix multiplication by the second kernel reshaped as in~\cref{eq:conv_output_collection}; and with 
\label{listitem:2}
\item independent of kernels operation described by~\cref{eq:conv_reduction} afterwards.
\end{enumerate*}
Therefore, coupled singular initialization can be applied to the product of matrices used in steps \ref*{listitem:1} and \ref*{listitem:2}.
We follow this approach and introduce funnels to the pairs of convolutions within the same ResNet block of denoising UNet.

\subsection{Pruning of temporal adaptors}
\label{supp:temporal_pruning}
\topic{Practical considerations.}
In the main text we described that we transform the update rule of the temporal block as follows $x_s + \hat{z} \left(1 - \alpha\right) r_t$, where $x_s$ and $r_t$ are  outputs of the spatial and temporal layers respectively, $\alpha$ is the learnable weight, and $\hat{z}$  is a zero-one gate multiplier.
Note that if the temporal block is pruned, \ie $\hat{z}_i  = 0$, then the gradient of the loss function  \wrt the temporal block's learnable parameters equals zero.  This affects the gradient momentum buffers used by optimizers.  For that reason, we do no update the momentum of the temporal block's parameters in case it has been pruned by all the devices at current iteration of multi-GPU training.

In the network, we parametrize the importance values $q_i$ with the sigmoid function with fixed temperature value of 0.1.
In the same way weight coefficients $\alpha_i$ were reparametrized.
For faster convergence, each value $q_i$ is initialized with the weight of the corresponding temporal block, \ie $1 - \alpha_i$.
We also found necessary to set the learning rate for the logits of importance values $q_i$ significantly higher than for the other parameters of the denoising UNet.

\topic{Constrained optimization.}
As discussed in the main text, we relate the importance values of temporal blocks $\left\lbrace q_i \right\rbrace_{i=1}^N$ to their inclusion probabilities for sampling without replacement $\left\lbrace p_i \right\rbrace_{i=1}^N$ by solving the following constrained optimization problem:
\begin{equation}
\begin{aligned}
    \min\limits_{c, \left\{p_i\right\}_i} \quad & \sum\nolimits_i \left(p_i - c q_i\right)^2,
    \\ \textrm{s.t.} \quad & \sum\nolimits_i p_i = n,
    \\   & 0 \le p_i \le 1.
\end{aligned}
\end{equation}
To solve it, we employ the common method of Lagrange multipliers assuming that all $q$-values are  strictly positive.
\Wlog we consider the case of sorted values $\left\{q_i\right\}_i$, \ie $q_1 \geq q_2 \geq \dots \geq q_N > 0$.
In detail, we define a Lagrangian 
\begin{equation}
\begin{aligned}
    \lhsisolated L\lft(c, \left\{p_i\right\}_i, \lambda, \beta, \left\{\gamma_i\right\}_i, \left\{\delta_i\right\}_i\rgt)  \\
      &= \lambda \sum\nolimits_i \left(p_i - c q_i\right)^2  \\
    &\quad +  \beta \left(\sum\nolimits_i p_i - n \right) \\
    &\quad +  \sum\nolimits_i \gamma_i \left(-p_i\right) \\
    &\quad +  \sum_i \delta_i \left(p_i - 1\right)
\end{aligned}
\end{equation}

and aim to solve the following system of equalities and inequalities
\begin{align}
    & \frac{\partial L}{\partial p_i} = 2 \lambda \left(p_i - c q_i\right) + \beta -\gamma_i + \delta_i = 0 \quad \forall i, \\
    & \frac{\partial L}{\partial c} = 2 \lambda \sum\nolimits_i \left(c q_i - p_i \right) q_i = 0, \\
    & \sum\nolimits_i p_i = n, \\
    & \gamma_i p_i = 0 \quad \forall i, \\
    & \delta_i \left(p_i - 1\right) = 0 \quad \forall i, \\
    & \gamma_i, \, \delta_i  \geq 0 \quad \forall i, \\
    & \lambda^2 + \beta^2 + \sum\nolimits_i \left(\gamma_i^2 + \delta_i^2\right) > 0.
\end{align}

\underline{Case $\lambda  = 0.$ }
In this case $\forall i \; \gamma_i  - \delta_i = \beta = \textrm{const}.$ If $\beta > 0,$ then $\forall i \; \gamma_i > \delta_i \geq 0 \Rightarrow \gamma_i > 0 \Rightarrow p_i = 0,$ which leads to a contradiction. Cases $\beta < 0$ and $\beta = 0$ also trivially lead to contradictions.

\underline{Case $\lambda  = 1.$ }
First, we derive that $c \sum\nolimits_i q_i^2 = \sum\nolimits_i p_i q_i,$ and thus
\begin{equation}
    c = \frac{\sum\nolimits_i p_i q_i}{\sum\nolimits_i q_i^2} > 0.
\end{equation}
Since $\forall i \; \frac{\partial L}{\partial p_i} = 0,$ then $\sum\nolimits_i q_i \frac{\partial L}{\partial p_i} = 0$.
\begin{align}
    \lhsisolated \sum\nolimits_i q_i \frac{\partial L}{\partial p_i} \\
    &= 2 \sum\nolimits_i q_i \left(p_i - c q_i \right) 
     + \beta \sum\nolimits_i q_i 
     - \sum\nolimits_i q_i \left(\gamma_i - \delta_i\right) \\ 
    &= \beta \sum\nolimits_i q_i 
     - \sum\nolimits_i q_i \left(\gamma_i - \delta_i\right) \\ 
    & = 0,
\end{align}
and therefore, 
\begin{equation}
    \beta = \frac{\sum\nolimits_i q_i \left(\gamma_i - \delta_i\right)}{\sum\nolimits_i q_i}.
    \label{eq:lagrange_beta_expression}
\end{equation}
\begin{lemma}
    For all indices $i$ it holds true that $\gamma_i = 0$.
\end{lemma}
\begin{proof}
Proof is given by contradiction.
Let us assume that $\exists k \; \gamma_k > 0 \Rightarrow p_k = 0 \Rightarrow \delta_k = 0 \Rightarrow -2c q_k + \beta - \gamma_k = 0 \Rightarrow \beta = 2c q_k + \gamma_k > 0$.
Then for any index $j$ such that $j > k$ and, consequently, $q_j \leq q_k$, the following equality holds true:
\begin{align}
    \lhsisolated 2\left(p_j - c q_j\right) + \beta - \gamma_j + \delta_j \\
    &= 2\left(p_j - c q_j\right) + 2c q_k + \gamma_k - \gamma_j + \delta_j \\
    &= 2 p_j + 2 c \left(q_k - q_j \right) + \left(\gamma_k - \gamma_j \right) + \delta_j \\
    &= 0.
\end{align}
All the terms in the last sum are known to be non-negative except for $\gamma_k - \gamma_j$. Therefore, $\gamma_k \leq \gamma_j \Rightarrow \gamma_j > 0 \Rightarrow p_j = 0$.
We define index $s$ as the largest index for which $\gamma_s = 0$, \ie
$\gamma_1 = \dots = \gamma_s = 0, \; \gamma_{s+1} > 0$. Note that $s > n$, since otherwise the equality $\sum\nolimits_i p_i = n$ cannot be satisfied.
Also, $p_j = 0$ for $j > s$.
Now we can rewrite \cref{eq:lagrange_beta_expression} as follows
\begin{align}
    \beta 
    &= \frac{1}{\sum\nolimits_i q_i} \left( - \sum\limits_{i:\,i \leq s} q_i \delta_i + \sum\limits_{i:\, i>s} q_i \gamma_i \right) \\
    &= \frac{1}{\sum\nolimits_i q_i} \left( - \sum\limits_{i:\,i \leq s} q_i \delta_i + \sum\limits_{i:\, i>s} q_i \left(\beta - 2 c q_i\right)\right) \\
    &= \frac{1}{\sum\nolimits_i q_i} \left( - \sum\limits_{i:\,i \leq s} q_i \delta_i -2c \sum\limits_{i:\, i>s} q_i^2\right) + \frac{\beta \sum\limits_{i:\, i>s} q_i}{\sum\nolimits_i q_i}.
\end{align}
After moving the last term from RHS to LHS, we obtain
\begin{equation}
    \frac{\beta \sum\limits_{i:\, i \leq s} q_i}{\sum\nolimits_i q_i} 
    =  \frac{1}{\sum\nolimits_i q_i} \left( - \sum\limits_{i:\,i \leq s} q_i \delta_i -2c \sum\limits_{i:\, i>s} q_i^2\right).
\end{equation}
Note that LHS is obviously strictly positive, while RHS is non-positive.
\end{proof}

Since $\forall i \; \gamma_i = 0$, we rewrite \cref{eq:lagrange_beta_expression},
\begin{equation}
    \beta = - \frac{\sum\nolimits_i q_i \delta_i }{\sum\nolimits_i q_i}.
\end{equation}
If $\forall i \; \delta_i = 0,$  then it is also true that $\beta = 0$, leading to $\forall i \; p_i = c q_i$. This is the case when inclusion probabilities are \emph{exactly} proportional to the importance values.
However, this is possible if and only if the maximum value $q_1$ is not too large in comparison with other values, since otherwise $p_1 > 1$.

In this last case $\exists k \; \delta_k > 0 \Rightarrow p_k = 1 \Rightarrow 2 \left(1 - c q_k \right) + \beta + \delta_k = 0 \Rightarrow \beta = 2 \left(c q_k - 1\right) - \delta_k. $
For any index $j$ such that $j < k$ we have
\begin{align}
    \lhsisolated 2\left(p_j - c q_j\right) + \beta  + \delta_j \\
    &= 2\left(p_j - c q_j\right) + 2 \left(c q_k - 1\right) - \delta_k  + \delta_j \\
    &= 2 p_j + 2 c \left(q_k - q_j \right) + \left(\delta_j - \delta_k \right) -2 \\
    &= 0.
\end{align}
By regrouping the terms, we obtain 
\begin{align}
    2 p_j + \delta_j  = 2c \left(q_j - q_k\right) + \delta_k + 2 > 2,
\end{align}
and since $p_j \leq 1$, this means that $\delta_j > 0 \Rightarrow p_j=1.$
Therefore, if for some index $k$ it turns out that $\delta_k > 0$, then for all smaller indices $j$, $\delta_j > 0$, and consequently $p_j=1$.
Again, let us define the index $t$ as the least index with zero $\delta$ coefficient, $\delta_{t-1} >0, \, \delta_t=\delta_{t+1}=\dots=0.$
Note that $t \leq n+1$, since more that $n$ inclusion probabilities cannot be equal to 1.

For $i \geq t, \; 2\left(p_i - c q_i\right) + \beta = 0 \Rightarrow p_i = c q_i - \frac{\beta}{2}.$

Therefore, 
\begin{equation}
    \sum\limits_i p_i = \sum\limits_{i:\, i<t} p_i + \sum\limits_{i:\, i \geq t} p_i = t-1 + \sum\limits_{i:\, i \geq t} \left(c q_i - \frac{\beta}{2}\right)  = n.
\end{equation}
Similarly,
\begin{equation}
    c \sum\limits_i q_i^2 = \sum\limits_i p_i q_i = \sum\limits_{i:\, i<t} q_i + \sum\limits_{i:\, i \geq t} q_i \left(c q_i - \frac{\beta}{2}\right)
\end{equation}

For any given $t=2,\dots,n$ two last equations allow us to compute the values of $c$ and $\beta$.
\begin{equation}
    \begin{pmatrix}
        \sum\limits_{i:\, i \geq t}q_i & -\left(N-t+1\right)
        \\ \\
        \sum\limits_{i:\, i<t}q_i^2 & \sum\limits_{i:\, i \geq t}q_i 
    \end{pmatrix}
    \cdot
    \begin{pmatrix}
        c \\ \\ \frac{\beta}{2}
    \end{pmatrix}
    = \begin{pmatrix}
        n - t + 1
        \\  \\ 
        \sum\limits_{i:\, i<t}q_i 
    \end{pmatrix},
\end{equation}
where $N$ is the total number of important values.
The solution exists and is unique for each $t$ since, obviously, \begin{equation*}
    \det \begin{pmatrix}
        \sum\limits_{i:\, i \geq t}q_i & -\left(N-t+1\right)
        \\ \\
        \sum\limits_{i:\, i<t}q_i^2 & \sum\limits_{i:\, i \geq t}q_i 
    \end{pmatrix} > 0.
\end{equation*}
In practice, we solve this matrix equation for each $2 \leq t \leq n$, test if the solution satisfies all the constraints, and after that select the solution that delivers the minimum value of our objective function.
At least one proper solution always exists, since for $t=n+1$ inclusion probabilities are equal to 1 for $n$ largest importance values and equal to 0 for all the rest indices. 

The solution of the system is differentiable \wrt all the $q_i$, leading to differentiable probabilities $p_i$.
However, as mentioned earlier, we use only $p_i$ computed with these equations for $i \geq t$, while for $i < t$ we set $p_i=1$. Therefore, we employ a straight-through estimator for these indices~\cite{bengio2013estimatingpropagatinggradientsstochastic}.

%% file: tables/tiny_ae.tex
\begin{table}[t]
\newcommand{\leftspace}{}
\centering
\resizebox{0.9\columnwidth}{!}{
\begin{tabular}{lrrcrrr} 
\toprule
\multirow{2}{*}{Decoder} 
& \multicolumn{2}{c}{FVD $\downarrow$} 
&
& \multicolumn{2}{c}{\jedi{} $\downarrow$} 
& \multirow{2}{*}{\makecell{Latency \\(ms) $\downarrow$}}\\
\cmidrule{2-3}\cmidrule{5-6}
& 25 FPS & 7 FPS & & 25 FPS & 7 FPS &\\
\midrule
Original decoder & 149 & 171 & & 1.07 & 1.21 & 1280\\
TAESD decoder & 149 & 179 & & 1.05 & 1.21 & 90\\
\bottomrule
\end{tabular}
}
\caption{
    \topic[0mm]{Impact of latent decoder.}
    While being significantly faster on device, decoder from TAESD has little to no impact on visual quality as measured by FVD and \jedi.
}
\label{tab:tiny_ae}
\vspace{-0.7em}
\end{table}